\documentclass[runningheads,envcountsect]{llncs}

\usepackage{eccv}

\usepackage{eccvabbrv}

\usepackage{graphicx}
\usepackage{booktabs}

\usepackage[accsupp]{axessibility}  

\usepackage{hyperref}

\usepackage{orcidlink}
\usepackage{algorithm}
\usepackage{algorithmic}
\usepackage{wrapfig}

\makeatletter
\renewcommand\paragraph{\@startsection{paragraph}{4}{\z@}%
  {0.6ex \@plus .5ex \@minus .2ex}%
  {-1em}%
  {\normalfont\normalsize\bfseries}}
\makeatother

\newcommand{\methodname}{\texttt{SVG-EAR}\xspace}

\begin{document}

\title{\texttt{SVG-EAR}: Parameter-Free Linear Compensation for Sparse Video Generation via Error-aware Routing} 
\titlerunning{SVG-EAR}

\newcommand{\equalcontrib}{\textsuperscript{\ensuremath{*}}}
\author{Xuanyi Zhou\equalcontrib \and
Qiuyang Mang\equalcontrib \and
Shuo Yang \and
Haocheng Xi \and \newline
Jintao Zhang \and
Huanzhi Mao \and
Joseph E. Gonzalez \and \newline
Kurt Keutzer \and
Ion Stoica \and
Alvin Cheung}

\authorrunning{X.~Zhou et al.}

\institute{UC Berkeley, USA\\
\equalcontrib Equal contribution.}

\maketitle

\begin{abstract}
  Diffusion Transformers (DiTs) have become a leading backbone for video generation, yet their quadratic attention cost remains a major bottleneck.
  Sparse attention reduces this cost by computing only a subset of attention blocks.
  However, prior methods often either drop the remaining blocks which incurs information loss, or rely on learned predictors to approximate them, introducing training overhead and potential output distribution shifting.
  In this paper, we show that the missing contributions can be recovered \emph{without training}: after semantic clustering, keys and values within each block exhibit strong similarity and can be well summarized by a small set of cluster centroids.
  Based on this observation, we introduce \methodname{}, a parameter-free linear compensation branch that uses the centroid to approximate skipped blocks and recover their contributions.
  While centroid compensation is accurate for most blocks, it can fail on a small subset.
  Standard sparsification typically selects blocks by attention scores, which indicate where the model places its attention mass, but not where the approximation error would be largest.
  \methodname{} therefore performs error-aware routing: a lightweight probe estimates the compensation error for each block, and we compute exactly the blocks with the highest error-to-cost ratio while compensating for skipped blocks. 
  We provide theoretical guarantees that relate attention reconstruction error to clustering quality, and empirically show that \methodname{} improves the quality-efficiency trade-off and increases throughput at the same generation fidelity on video diffusion tasks.
  Overall, \methodname{} establishes a clear Pareto frontier over prior approaches, achieving up to $1.77\times$ and $1.93\times$ speedups while maintaining PSNRs of up to $29.759$ and $31.043$ on Wan2.2 and HunyuanVideo, respectively.
\end{abstract}
\section{Introduction}
\label{sec:intro}

\begin{figure}[h]
  \centering
    \includegraphics[width=1\textwidth]{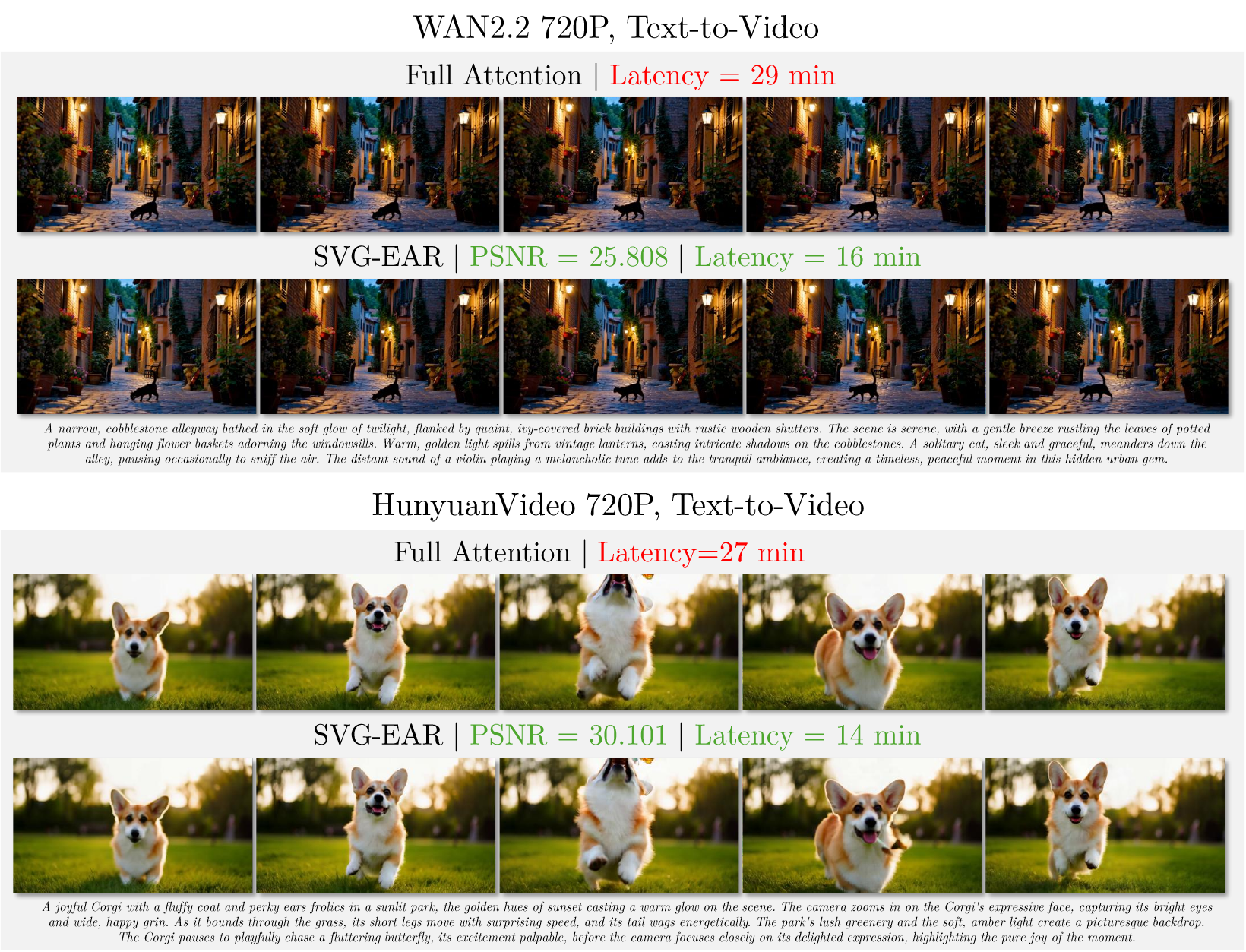}
  \caption{\methodname significantly accelerates video generation for Wan2.2 and HunyuanVideo. On a single NVIDIA H100 GPU, achieving $1.81\times$ and $1.93\times$ speedups with $26$ and $30$ PSNR, respectively. }
  \label{figure:teaser}
\end{figure}

Diffusion transformers (DiTs) have become a dominant paradigm for high-fidelity image and video generation~\cite{kong2024hunyuanvideo, wan2025wan}.
However, in video generation settings, the token sequence length grows rapidly with resolution and frame count, while the quadratic cost of attention quickly becomes a primary bottleneck.
As a result, accelerating attention \emph{without} noticeable quality degradation is increasingly critical for long-form and high-resolution video diffusion~\cite{xi2025sparsevideogenacceleratingvideo, yang2025sparse, zhang2025spargeattention}.

A large body of work tackles this bottleneck via sparse attention.
A particularly practical family exploits \emph{structured redundancy} in attention maps:
tokens are semantically clustered and permuted so that similar tokens become contiguous in memory, turning the attention matrix into a block structure over query clusters and key clusters.
Then, only a subset of query--key blocks are computed exactly (block-sparse), while the rest are ignored or approximated, yielding hardware-friendly sparsity patterns and substantial speedups~\cite{yang2025sparse}.
This ``cluster--permute--route'' pipeline is appealing in video DiTs because it aligns well with efficient kernels and preserves important structure.

Despite these advances, existing methods face a fundamental tension in block selection and the treatment of unselected blocks.
Many approaches select blocks by approximated attention scores (\emph{e.g.,} top-$k$/top-$p$) and \emph{drop} low-score blocks entirely.
While intuitive, low-score blocks can still collectively carry important global context (\emph{e.g.,} background consistency, long-range semantic coupling, and weak but numerous dependencies).
Consequently, naively discarding low-score blocks incurs non-trivial information loss and can manifest as perceptible quality degradation in video generation.
Recent works such as SLA~\cite{zhang2025sla} and SLA2~\cite{zhang2026sla2} address this by pairing sparse exact attention with a learned linear branch that approximates the contributions of dropped blocks, recovering much of the lost context.
However, this approach introduces additional trainable parameters and requires fine-tuning, limiting its plug-and-play applicability and sacrificing the fidelity.

A natural solution is to avoid pure dropping by leveraging the high within-block similarity revealed by clustering.
For blocks that are not selected for exact computation, one can apply a \emph{parameter-free linear compensation} using centroids:
replace queries and keys within a cluster with their centroid and approximate the contribution of an entire block with a shared interaction.
This branch requires no training and no additional parameters, and it mitigates information loss from dropped blocks.

However, linear compensation alone does not resolve the core issue.
Crucially, once a compensation branch exists, conventional score-based routing becomes \emph{misaligned} with the objective of controlling the final approximation error.
A high-score block may be highly coherent within the cluster and thus well approximated by its centroid, making exact computation unnecessary.
Conversely, a low-score block can contain diverse key-value interactions where centroid-based compensation induces substantial error.
Therefore, under a fixed compute budget, the goal should not be to preserve the highest-score blocks, but to \emph{minimize the reconstruction error} between full attention and its compensated counterpart, allocating exact computation to the blocks where compensation fails.

Motivated by this insight, we propose \methodname{}, a sparse attention method that combines \emph{parameter-free linear compensation} with \emph{error-aware routing}.
\\ \methodname{} first clusters queries and keys to form a block structure based on similarity.
It then computes exact attention for a subset of blocks under the density budget assigned by users, and approximates the remaining blocks using linear compensation with key/value cluster centroids.
Unlike prior score-based selection, \methodname{} estimates the \emph{compensation error} of each block using a lightweight probing procedure and greedily selects blocks with the highest \emph{error-to-cost ratio} (\emph{i.e.,} the predicted compensation error normalized by the block size) under a density budget. 
To make routing practical at inference time, we exploit intra-cluster similarity and use query centroids as proxies for individual queries, reducing the estimation cost from quadratic $\mathcal{O}(N_q N_k d)$ to near-linear $\mathcal{O}(C_q N_k d)$.
We further implement a fused, streaming kernel to avoid materializing intermediate logits and to keep the routing overhead negligible in practice.

We validate the proposed design both theoretically and empirically.
On the theory side, we provide an upper bound that relates the true attention-map error to our estimated error, characterizing its dependence on clustering quality and sequence length.
On the empirical side, \methodname{} achieves a superior error--density trade-off on video generation: at the same density, it reduces attention reconstruction error and improves generation quality, and at the same quality target it operates at lower density and delivers higher throughput (see \S\ref{sec:experiments}).
Specifically, \methodname{} establishes a clear Pareto frontier over prior approaches~\cite{yang2025sparse, xi2025sparsevideogenacceleratingvideo,zhang2025spargeattention}, reaching up to $1.77\times$ and $1.93\times$ speedups while maintaining PSNRs of up to $29.759$ (Wan2.2~\cite{wan2025wan}) and $31.043$ (HunyuanVideo~\cite{kong2024hunyuanvideo}), respectively.
Overall, our results suggest that when compensation is available, the key to high-fidelity sparse attention is not ``\emph{selecting high-score blocks}'', but identifying where compensation breaks and prioritizing those blocks for exact computation.

We summarize our key contributions as follows:
\begin{enumerate}
    \item We identify two fundamental misalignments in score-based sparse attention:
    (i) naively \emph{dropping} low-attention-score blocks can cause substantial information loss; and (ii) once an approximation/compensation branch is introduced, block selection should \emph{not} prioritize high-score blocks, but instead prioritize blocks that would otherwise incur \emph{large approximation error}.
    \item We prototype and implement a compensation-and-routing sparse attention mechanism: a parameter-free \emph{linear compensation} branch that recovers contributions from uncomputed blocks via cluster means, and an \emph{error-aware routing} strategy that, under a fixed budget, identifies and computes the blocks with the largest induced error, yielding a markedly improved error--density trade-off.
    \item We further translate the design into an end-to-end system with efficient kernels and execution flow that keep the overhead negligible in practice, delivering substantial and consistent speedups on real video generation workloads while maintaining generation fidelity.
\end{enumerate}

\section{Related Work}
\label{sec:related}

\paragraph{Sparse Attention for Video Generation.}
Sparse attention is a leading direction for accelerating video Diffusion Transformers~\cite{zhangefficient} because 3D spatiotemporal self-attention scales quadratically with the huge token sequence length. Sparse attention mechanisms for video generation can be roughly divided into static and dynamic approaches, based on whether the sparsity pattern is fixed offline or determined online at inference time. Static approaches typically exploit recurring structural patterns~\cite{xi2025sparsevideogenacceleratingvideo, li2025radialattentiononlogn, zhao2025paroattention, jiang2024miference, chen2025rettention, zhang2025ditfastattnv2}. Dynamic approaches infer masks or critical-token sets per sample and timestep, commonly introducing an identification/proxy-scoring step whose overhead is usually designed to be negligible~\cite{yang2025sparse, xia2025trainingfreeadaptivesparseattention, chen2025rainfusion2, liu2026mixture, sun2025vorta, wu2025vmoba, cai2025mixture, zhang2025spargeattention, li2026pisa, li2025psa}. Beyond training-free methods, newer ``trainable sparse attention'' lines push sparsity higher with fine-tuning and additional branch compensation \cite{zhang2025vsa, zhang2025sla, zhang2026sla2, zhang2026spargeattention2, wu2025usv, liang2026vmonarch, agarwal2026monarchrt}.

\paragraph{Linear Attention for Diffusion Models.}
Linear attention and state-space alternatives are important for video diffusion because quadratic attention quickly dominates latency and memory as spatiotemporal token counts grow. Representative families can be grouped into kernelized linear attention~\cite{chen2025sana}, state-space/SSM-centered designs that use structured recurrence for global mixing~\cite{gao2024matten, wang2025lingen, huang2025m4v}, and hybrid local–global designs that explicitly balance a cheap linear surrogate with selective higher-fidelity pathways~\cite{fang2026salad, zhang2025sla}. Though these methods enable predictable memory scaling and improved feasibility of long-context generation under fixed budgets, they still suffer from low quality and long-term dependencies.

\paragraph{Other Optimization Techniques.}
Cache-based acceleration exploits redundancy across denoising steps or between conditional/unconditional branches in classifier-free guidance~\cite{ma2024deepcache, ma2024learning, lv2024fastercache, kahatapitiya2025adaptive, ma2025magcache, chu2025omnicache, bu2025dicache, liu2025timestep}. Parallelization becomes essential when targeting deployment-scale throughput or higher resolutions~\cite{li2024distrifusion, feng2025streamdiffusionv2, fang2024pipefusion, fang2024xdit, jacobs2023deepspeed, liu2023ring, chu2025usp}. Quantization reduces memory bandwidth and leverages low-precision tensor cores to accelerate further the linear module~\cite{wu2024ptq4dit, li2025dvd, li2024svdquant} or attention module~\cite{zhang2025sageattention, zhang2024sageattention2, zhang2025sageattention2++, zhang2025sageattention3, zhang2026sagebwd}.
Distillation and few-step sampling compress long diffusion trajectories into a small number of steps, usually leading to more than $10\times$ speedup~\cite{lu2025simplifyingstabilizingscalingcontinuoustime,yin2024one,yin2024improved,zheng2025large}.

\paragraph{Autoregressive and Streaming Video Generation.} Autoregressive (streaming) video diffusion generates frames (or chunks) sequentially and enables efficient KV caching, usually accompanied by unbounded temporal horizons and interactive control. CausVid~\cite{yin2025slow} and Self-Forcing~\cite{huang2025self} exemplify this shift by converting bidirectional video diffusion into causal generation that better matches interactive inference and is consolidated through further algorithmic innovations to mitigate anti-drifting~\cite{liu2025rolling,chen2026context,zhu2026causal,lv2026light,yang2025longlive}. World Models~\cite{xi2026quant,hyworld2025,worldplay2025,hunyuanworld2025tencent,bruce2024genie,gao2025longvie} objectives overlap with streaming video generation but extend it toward closed-loop simulation, where an agent or user can intervene at each step.

\section{Motivation}
\label{sec:motivation}

\begin{figure*}[t]
    \centering
    \includegraphics[width=0.9\textwidth]{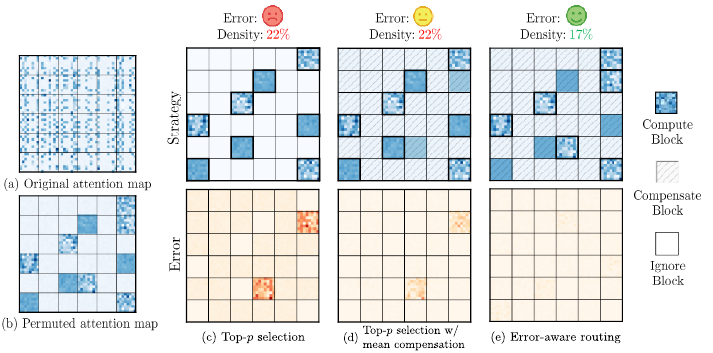}
    \caption{Existing methods and top-$p$ selection fall short: \textit{dropping "low-score" blocks} and \textit{error-unaware block selection} degrade the error–density trade-off. (a) Original attention map.
    (b) Permuted map after semantic-aware clustering.
    (c) Ignoring low-score blocks causes a large, sparse attention error.
    (d) Linear compensation with cluster means still yields high error due to naive top-$p$ selection.
    (e) Our method improves both error and density by routing based on the gap between full computation and compensation.}
    \label{fig:motivation}
\end{figure*}

\subsection{Structured Redundancy in Attention Computation}
Previous works exploit the inherent sparsity in attentions to accelerate DiTs~\cite{xi2025sparsevideogenacceleratingvideo, yang2025sparse, zhang2025spargeattention, zhao2025paroattention}.
We first cluster semantically similar tokens and permute them so that tokens within each cluster are laid out contiguously.
This allows the attention map to be partitioned into blocks that often exhibit high inner-block similarity.
As shown in \cref{fig:motivation}~(c), the existing methods then select a subset of these blocks for exact attention ranked by their approximated attention scores, while simply ignoring the low score blocks.
However, this approach can lead to significant information loss, as the ignored blocks can also contain important contextual information, and the limited budget may not allow for their inclusion.

Fortunately, the similarities within each block naturally provide an opportunity to efficiently approximate the skipped blocks' contributions without any training and additional parameters.
As shown in \cref{fig:motivation}~(d), a simple strategy here is to use the mean of the remaining blocks in each cluster to fill in the gaps left by the ignored blocks.
Given a block to compensate, suppose the mean of the key-value tokens is $\bar{k}$.
The attention logits for all entries in the block interacted with a query $q_i$ can be then computed as $q_i\bar{k}^{T}/\sqrt{d}$, where $d$ is the dimensionality of the token vectors.

\subsection{Existing Block Selection Fails to Control Output Error}
While linear compensation improves the error–density trade-off, its effectiveness ultimately depends on how the blocks are selected.
Existing block selection strategies are score-based and thus misaligned with the objective of controlling the final output error when linear compensation is introduced.
For example, a block with high attention scores may exhibit strong inner-block similarity, making it well suited for linear compensation.
In contrast, a block with low attention scores can contain diverse key–value interactions, where linear compensation induces substantial approximation error.
To maximize fidelity under a fixed compute budget, the goal should be to minimize the reconstruction error of the attention map.
As shown in \cref{fig:motivation}~(e), our proposed approach routes computation based on the gap between the full attention map and its linearly compensated counterpart.
As a result, by directly optimizing an error-aware objective and prioritizing the blocks that contribute most to the reconstruction error, we achieve a superior error–density trade-off, as detailed below.

\section{Methodology}
\label{sec:methodology}

In this section, we first formulate the problem based on our parameter-free linear compensation.
We then introduce \methodname, an error-aware routing method with theoretical guarantees and low overhead.
Our key insight is that the error between the full attention map and its linearly compensated counterpart can be accurately estimated via a lightweight probing algorithm, which enables precise attention routing.
We further show how the linear compensation leverages value-cluster means to produce the final output efficiently.
Finally, we implement a fused kernel that estimates the error using a streaming update.

\subsection{Problem Formulation}
Given a transformer attention layer with query tokens $Q \in \mathbb{R}^{N_q \times d}$ and key tokens $K \in \mathbb{R}^{N_k \times d}$. 
We cluster $Q$ and $K$ using flash k-means~\cite{yang2025sparse}, and obtain $C_q$ and $C_k$ clusters, respectively.
For each query $q_i$ and key $k_j$, we use $\bar{q}_i$ and $\bar{k}_j$ to denote the mean token in the cluster that contains $q_i$ and $k_j$, respectively.
We denote $\bar{Q} \in \mathbb{R}^{N_q \times d}$ and $\bar{K} \in \mathbb{R}^{N_k \times d}$ as the matrices of per-token cluster means for $Q$ and $K$, respectively; that is, the $i$-th row of $\bar{Q}$ equals $\bar{q}_i$ (the centroid of the cluster containing $q_i$), and similarly the $j$-th row of $\bar{K}$ equals $\bar{k}_j$.

\begin{figure*}[t]
    \centering
    \includegraphics[width=1\textwidth]{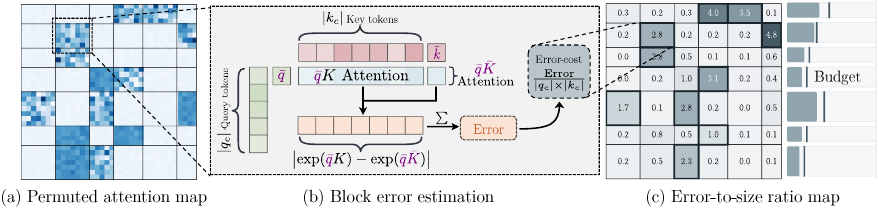}
    \caption{Overview of \methodname (a) The attention map after semantic-aware clustering.
    (b) Block error estimation. Using the cluster mean as a proxy for individual queries, the block error is computed via the sum of squared differences between exponentiated logits of individual keys $k_j$ and the key mean $\bar{k}$, then normalized by the block area to determine the error-to-size ratio.
    (c) Blocks with the highest error-to-size ratio are greedily selected for exact attention within the budget, while the rest are assigned to linear compensation.}
    \label{fig:methodology}
\end{figure*}

Our core idea here is to approximate part of the attention computation linearly with cluster means.
For a query $q_i$, when approximating its interaction with a cluster of keys $k_c = \{k_j\}$, we replace each key $k_j$ in the logit computation with the corresponding cluster mean $\bar{k}_j$. 
This gives an approximate logit ${q_i \bar{k}_j^T} / {\sqrt{d}}$, which is shared across all keys within the same cluster.

We use a binary routing mask $M \in \{0,1\}^{N_q \times N_k}$ for each query $q_i$.
If $M^{(i)}_j=1$, we compute the interaction between $q_i$ and $k_j$ exactly.
If $M^{(i)}_j=0$, we approximate it using the cluster mean.
To match ML accelerator-friendly sparsity, where the tokens in the same cluster will be permutated continuously and processed together, the mask should be applied on the query-key block level:
All queries in the same cluster share the same mask.
For a fixed query, all keys in the same cluster share the same mask value.
With this mask, we write the sparse attention map as
\begin{equation}
    \tilde{A}_{\text{sparse}}^{(M)} = \operatorname{softmax}\Big(\frac{Q K^T}{\sqrt{d}}\odot M + \frac{Q \bar{K}^T}{\sqrt{d}}\odot\big(1-M\big)\Big)
\end{equation}

Given a budget $\rho$ (\emph{i.e.,} density), our goal is to find a mask $M$ that minimizes the mean squared error (MSE) between the full attention map and the sparse attention map, subject to the constraint that the number of non-zero entries in $M$ is at most $\rho N_q N_k$.
Formally, we formulate the optimization problem as:
\begin{equation}
    \min_{M} \|\tilde{A}_{\text{sparse}}^{(M)} - A_{\text{full}}\|^2_F \quad \text{s.t.} \quad \|M\|_F^2 \leq \rho N_q N_k,
    \label{eq:optimization_problem}
\end{equation}
where $\|\cdot\|_F$ denotes the Frobenius norm and $A_{\text{full}} = \operatorname{softmax}\Big(\frac{Q K^T}{\sqrt{d}}\Big)$ is the full attention map.

\subsection{Error-Aware Routing}
\paragraph{Relaxed optimization problem.}
The above optimization problem is difficult to solve efficiently at routing time because the softmax normalizer couples all keys for each query, and $A_{\text{full}}$ and $\tilde{A}_{\text{sparse}}^{(M)}$ use different normalization terms. In particular, for any query, altering the routing decision for any single key cluster changes the normalizer and thus changes the attention weights assigned to all other keys.
Consequently, selecting an optimal block-level mask cannot be reduced to independent per-block choices; in the worst case it requires searching over a combinatorial space of size up to $\mathcal{O}(2^{C_q \cdot C_k})$, which is intractable.

To obtain a practical and block-separable proxy, we therefore relax the objective by focusing on the exponentiated logits as follows:
\begin{equation}
    \min_{M} \sum_{i, j} \big( 1 - M^{(i)}_j \big) \big( \exp(\frac{q_i \bar{k}_j^T}{\sqrt{d}}) - \exp(\frac{q_i k_j^T}{\sqrt{d}}) \big)^2 \quad \text{s.t.} \quad \|M\|_F^2 \leq \rho N_q N_k.
\end{equation}

Intuitively, this relaxation shifts the objective from finding a mask $M$ that minimizes the discrepancy between two normalized attention maps to minimizing the discrepancy between their unnormalized attention weights.
Although this relaxation may fail to identify masks whose near-optimality only emerges after normalization, it still provides a useful proxy for routing decisions.
This is because, in our setting, the variation in the softmax normalizer is practically limited for two reasons:
(1) semantic clustering ensures that each $k_j$ is close to its centroid $\bar{k}_j$, reducing deviation in the logits;
(2) the unmasked tokens typically receive high attention scores, so they dominate the normalization term, thereby limiting the impact of the remaining masked tokens on the normalizer.

\paragraph{Oracle solution.}
For the relaxed optimization problem, an oracle solution can in principle be obtained by selecting the query--key blocks that incur the largest squared errors under the budget constraint, which reduces to a tractable $0$--$1$ knapsack problem.
Specifically, let
\begin{equation}
    \epsilon_{i,j}^2 = \big( \exp\big(\frac{q_i \bar{k}_j^T}{\sqrt{d}}\big) - \exp\big(\frac{q_i k_j^T}{\sqrt{d}}\big) \big)^2
\end{equation}
denote the squared error for the entry $(i,j)$.
Each query-key block $(q_c, k_c)$ can then be treated as an item in the knapsack formulation, where the value is given by
$\sum_{(i,j)\in(q_c,k_c)} \epsilon_{i,j}^2$,
and the weight corresponds to its size $|q_c| |k_c|$.
After that, the problem can be solved using standard knapsack algorithms or well approximated by greedy methods that select blocks in descending order of error-to-cost ratio $\frac{\sum_{(i,j)\in(q_c,k_c)} \epsilon_{i,j}^2}{|q_c| |k_c|}$ until the budget is exhausted.

The oracle solution is fully error-aware; however, it is not practical. Computing the exact per-entry error requires evaluating the full attention logits for all token pairs, which incurs the same $\mathcal{O}(N_qN_kd)$ quadratic cost as full attention itself. In other words, the oracle depends on precisely the information that sparse attention is designed to avoid computing. 

\paragraph{Error estimation.}
To achieve error-aware routing without incurring the prohibitive cost of computing the exact error, we propose to estimate the error using a lightweight probing algorithm.
Again, we leverage the inherent similarity within each query token cluster, which allows us to use the cluster mean as a proxy for the individual queries in the error estimation.
Specifically, we define the following estimated error for each entry $(i,j)$:
\begin{equation}
    \label{eq:estimated_error}
    \hat{\epsilon}_{i,j}^2 = \big( \exp\big(\frac{\bar{q}_i \bar{k}_j^T}{\sqrt{d}}\big) - \exp\big(\frac{\bar{q}_i k_j^T}{\sqrt{d}}\big) \big)^2
\end{equation}
Because queries within the same cluster share a common estimated error, the computational complexity reduces to $\mathcal{O}(C_q N_k d)$, which is negligible in practice when an optimized kernel implementation is used.

We then use the estimated errors to determine whether each query–key block is assigned to exact attention or linear compensation.
Specifically, we first aggregate the estimated errors within each block. We then greedily select the blocks with the highest error-to-size ratio $ {\sum_{(i,j)\in(q_c,k_c)} \hat{\epsilon}_{i,j}^2} / {|q_c| |k_c|} $ for exact attention until the budget is exhausted, while assigning the remaining blocks to linear compensation. 

\begin{proposition}
Let 
$\delta_q^2 = \frac{1}{N_q} \sum_i \| q_i - \bar{q}_i \|_2^2$
denote the average squared $\ell_2$ error between each query and its cluster mean, and let $K_{\max}$ denote the maximum $\ell_2$ norm of the key tokens. 
Given any mask $M$ that does not significantly perturb the attention normalizers, we have:
\begin{equation}
    \label{eq:mse_bound}
    \frac{1}{N_qN_k} \|\tilde{A}_{\text{sparse}}^{(M)} - A_{\text{full}}\|^2_F \;\leq\; \frac{2}{N_qN_k} \sum_{i,j} (1-M^{(i)}_j)\, \frac{\hat{\epsilon}_{i,j}^2}{Z_i^2} \;+\; \frac{8\delta_q^2K_{\max}^2}{N_k d}
\end{equation},
where $Z_i$ denotes the softmax normalizer for query $i$.
\end{proposition}
See Appendix~\ref{app:bound} for the proof and the formal statement of the normalizer stability assumption.
The left-hand side is the true MSE on the attention map, while the right-hand side consists of the estimated MSE (using $\hat{\epsilon}_{i,j}^2$) scaled by a constant factor, plus a residual term.
The residual $\frac{8\delta_q^2 K_{\max}^2}{N_k d}$ scales linearly with the clustering error $\delta_q^2$ and decreases inversely with the sequence length $N_k$.
Therefore, as clustering improves (\emph{i.e.,} $\delta_q^2 \to 0$) or the sequence length increases, the bound becomes asymptotically tight, establishing that our error estimation is both theoretically safe and quantitatively controllable.


\subsection{Linear Compensation}
After determining the mask $M$, we compute the final attention output by combining exact and approximated contributions.
Let $V \in \mathbb{R}^{N_k \times d}$ denote the value tokens, which follow the same clustering assignment as $K$, and let $\bar{V} \in \mathbb{R}^{N_k \times d}$ denote the matrix of mean value tokens, where each $\bar{v}_j$ is the average of $V$ over the key cluster containing $v_j$.
The sparse attention output is then given by:
\begin{equation}
\tilde{O}_{\text{sparse}}^{(M)} = \big(\tilde{A}_{\text{sparse}}^{(M)} \odot M\big) V + \big(\tilde{A}_{\text{sparse}}^{(M)} \odot (1-M)\big) \bar{V}
\end{equation}
For blocks where $M^{(i)}_j=1$, the output is computed with the original values $V$.
For blocks where $M^{(i)}_j=0$, since all keys within the cluster share the same approximate attention weight, the contribution of a key cluster $k_c$ reduces to $|k_c| \cdot \tilde{A}_{\text{sparse}}^{(M)}(i, j) \cdot \bar{v}_j$.
This significantly improves computational efficiency: for a given query and masked block, 
we only need to compute a single logit using the cluster mean $\bar{k}_j$ 
and a single scalar–vector product with the mean value $\bar{v}_j$.
Hence, the overall time complexity of \methodname's attention computation is reduced to 
$\mathcal{O}\Big((N_q C_k + \rho N_k N_q)\cdot d\Big)$.

Moreover, since the linear compensation also considers the value tokens, we can further improve the error estimation above by incorporating each $v_j$ and $\bar{v}_j$ into the analysis.
A \textit{value-aware error estimation} is thus given by:
\begin{equation}
    \label{eq:estimated_error_value}
    \tilde{\epsilon}_{i,j}^2 = \| \exp\big(\frac{\bar{q}_i \bar{k}_j^T}{\sqrt{d}}\big)\bar{v}_j - \exp\big(\frac{\bar{q}_i k_j^T}{\sqrt{d}}\big)v_j \|^2_2
\end{equation}
which can be computed in the same complexity $\mathcal{O}(C_q N_k d)$ as Equation~\ref{eq:estimated_error} while providing a practically more accurate estimation.

\begin{figure}[t]
  \centering
    \includegraphics[width=1\linewidth]{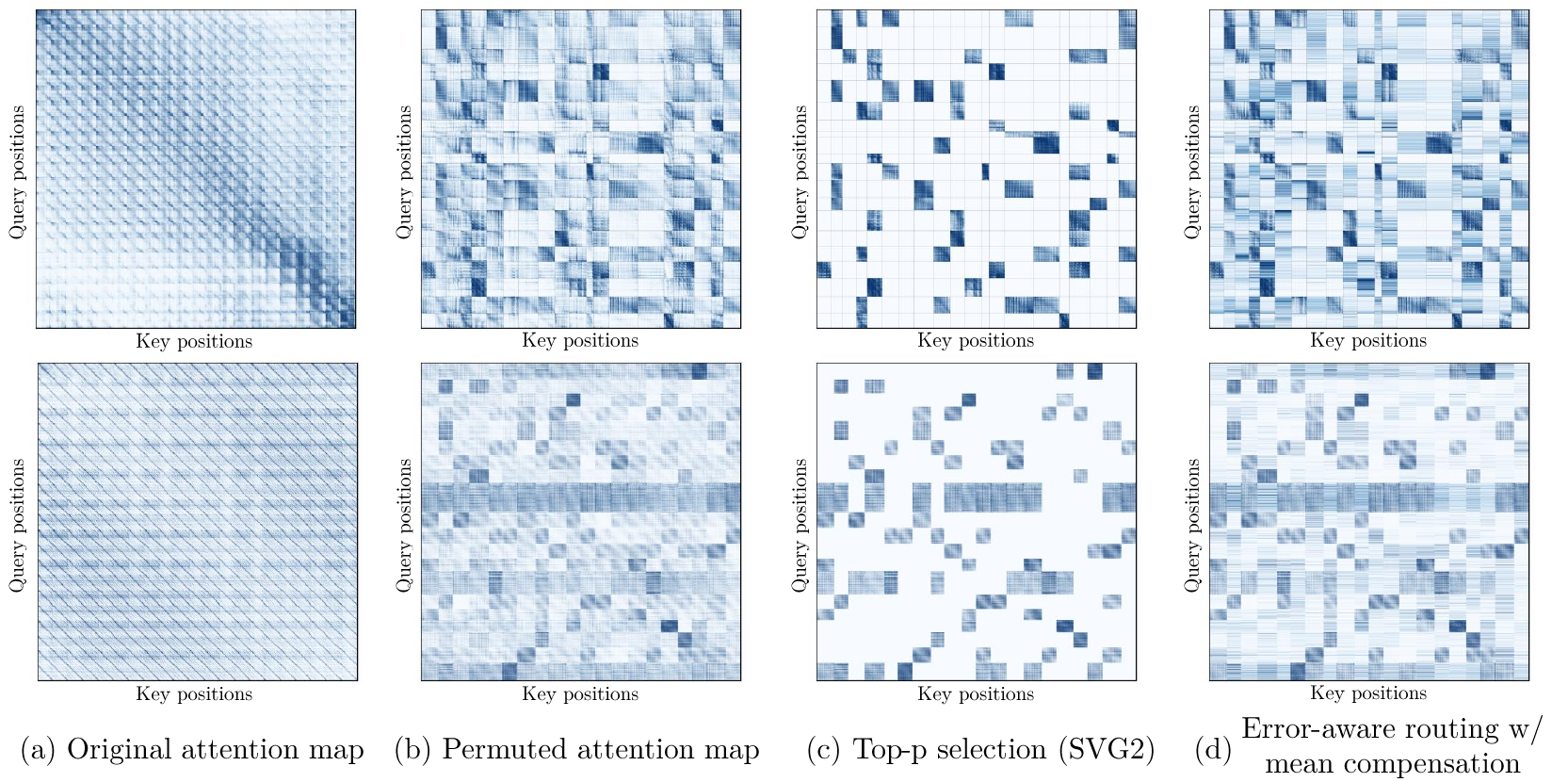}
  \caption{Visualization of attention maps from specific attention heads in Wan2.2 during text-to-video generation. (a) Original attention maps with sparse patterns, using heatmaps to represent attention weights. (b) Permuted attention maps. Following SVG2, we permute the attention maps such that attention begins to cluster into specific blocks. (c) Applying SVG2's top-$p$ selection to these permuted attention maps, which ignores the sparse attention mechanisms of the remaining blocks. (d) Applying our Error-aware Routing and Mean Compensation mechanisms to the permuted attention maps, achieving higher attention map similarity.}
  \label{figure:visualization}
\end{figure}

\subsection{Efficient Kernel Implementation}
The error estimation procedure is inherently memory I/O–bound, as naively materializing all logits requires repeatedly writing intermediate results to HBM. 
To address this bottleneck, we design a custom kernel based on a streaming update scheme.
Specifically, instead of materializing all logits in HBM and computing the squared errors afterward, we fuse these steps by expanding the squared terms and maintaining the required running statistics on the fly. 
To ensure numerical stability, we subtract the running maximum from the logits prior to exponentiation. 
This design reduces HBM accesses to $\mathcal{O}(N_k C_q d \cdot \mathbf{S}^{-1})$, 
where $\mathbf{S}$ denotes the SRAM size, while incurring only a small constant-factor increase in FLOPs within the same $\mathcal{O}(N_k C_q d)$ complexity.
For the linear compensation part, we reuse SVG2's~\cite{yang2025sparse} customized attention kernel that accepts dynamic block sizes as inputs. 
See Appendix~\ref{app:kernel} for more implementation details.

\section{Experiments}
\label{sec:experiments}

\begin{table*}[t]
\centering
\caption{Quality and efficiency benchmarking results of \methodname and baselines, where the best results are highlighted, and the second-best results are underlined.}
\label{table:main-results}
\renewcommand{\arraystretch}{1.5}
\resizebox{1\linewidth}{!}{%
\begin{tabular}{l|ccccc|ccc}
\toprule
\textbf{Config} & \textbf{PSNR}$\uparrow$ & \textbf{SSIM}$\uparrow$ & \textbf{LPIPS}$\downarrow$ & \textbf{ImgQual}$\uparrow$ & \textbf{SubCons}$\uparrow$ & \textbf{Density}$\downarrow$ & \textbf{FLOP}$\downarrow$ & \textbf{Speedup}$\uparrow$ \\
\midrule

\textit{Wan 2.2 14B, 720P, I2V} & - & - & - & 0.704 & 0.960 & 100\% & 658.46 PFLOPS & $1\times$ \\
\cmidrule(lr){1-9}
SpargeAttn             & 27.140 & 0.883 & 0.116 & \underline{0.703} & 0.958 & 30.15\% & 396.83 PFLOPS & $1.58\times$ \\
SVG                    & 25.297 & 0.844 & 0.139 & 0.703 & \underline{0.958} & 30.25\% & 397.20 PFLOPS & $1.58\times$ \\
SVG2                   & 27.668 & 0.888 & 0.117 & 0.701 & 0.958 & 29.38\% & 393.95 PFLOPS & $1.61\times$ \\
\rowcolor{lightblue}
\methodname     & \textbf{29.759} & \textbf{0.918} & \textbf{0.093} & \textbf{0.704} & \textbf{0.959} & \underline{23.64\%} & 378.88 PFLOPS & \underline{1.61$\times$} \\
\rowcolor{lightblue}
\methodname-Turbo & \underline{28.344} & \underline{0.900} & \underline{0.108} & 0.702 & 0.958 & \textbf{20.42\%} & 363.85 PFLOPS & \textbf{1.77$\times$} \\
\midrule

\textit{Wan 2.2 14B, 720P, T2V} & - & - & - & 0.706 & 0.916 & 100\% & 658.46 PFLOPS & $1\times$ \\
\cmidrule(lr){1-9}
SpargeAttn             & 20.872 & 0.708 & 0.242 & \underline{0.708} & \textbf{0.916} & 30.15\% & 396.83 PFLOPS & $1.58\times$ \\ 
SVG      & 19.455 & 0.654 & 0.292 & \textbf{0.712} & 0.912 & 30.25\% & 397.20 PFLOPS & $1.59\times$ \\
SVG2     & 23.556 & 0.802 & 0.183 & 0.705 & 0.914 & 32.30\% & 404.88 PFLOPS & $1.57\times$ \\
\rowcolor{lightblue}
\methodname     & \textbf{24.995} & \textbf{0.841} & \textbf{0.153} & 0.706 & \underline{0.915} & \underline{25.95\%} & 387.53 PFLOPS & \underline{$1.59\times$} \\
\rowcolor{lightblue}
\methodname-Turbo & \underline{23.940} & \underline{0.814} & \underline{0.174} & 0.705 & 0.915 & \textbf{22.25\%} & 370.71 PFLOPS & \textbf{1.75$\times$} \\
\midrule

\textit{Hunyuan 13B, 720P, T2V} & - & - & - & 0.665 & 0.904 & 100\% & 612.38 PFLOPS & $1\times$ \\
\cmidrule(lr){1-9}
SpargeAttn             & 24.589 & 0.796 & 0.232 & 0.629 & \textbf{0.908} & 40.09\% & 389.76 PFLOPS & $1.38\times$ \\
SVG      & 27.325 & 0.880 & 0.140 & \textbf{0.665} & \underline{0.905} & 29.92\% & 351.97 PFLOPS & $1.57 \times$ \\
SVG2      & \underline{29.445} & \underline{0.911} & \underline{0.112} & 0.654 & 0.901 & 26.21\% & 299.02 PFLOPS & \underline{1.89$\times$} \\
\rowcolor{lightblue}
\methodname      & \textbf{31.043} & \textbf{0.928} & \textbf{0.092} & \underline{0.659} & 0.903 & 22.17\% & 281.86 PFLOPS & \textbf{1.93$\times$} \\

\bottomrule
\end{tabular}
}
\end{table*}


\paragraph{Models.} We evaluate \methodname on open-sourced state-of-the-art video generation models, including Wan2.2-I2V/T2V-A14B~\cite{wan2025wan}, and HunyuanVideo-T2V-13B~\cite{kong2024hunyuanvideo} to generate videos with 720p resolution. After being tokenized by 3D-VAE, Wan2.2 generates 21 frames with 3,600 tokens per frame, while HunyuanVideo processes 33 frames with 3,600 tokens per frame.

\paragraph{Metrics.} 
We assess the similarity of the generated video compared to full attention using the following metrics: Peak Signal-to-Noise Ratio (PSNR), Learned Perceptual Image Patch Similarity (LPIPS), and Structural Similarity Index Measure (SSIM). We use VBench~\cite{huang2024vbench} to evaluate the video quality. To quantify the efficiency of sparse attention mechanisms (\emph{i.e.,} computational budget), we use density, which is defined as the sparse attention computation divided by the full attention computation. To assess end-to-end efficiency, we test the Speedup Ratio, which is calculated as the inference time of the full attention divided by the inference time of the sparse attention.

\paragraph{Datasets.} For text-to-video generation, we adopt the prompt in Penguin Benchmark after prompt optimization provided by the VBench team. For image-to-video generation, we adopt the prompt-image pairs provided by VBench with prompt enhancement. We evaluate the video quality on a subset of 50 samples randomly selected from text-to-video and image-to-video datasets, respectively. 

\paragraph{Baselines.} We compare \methodname against state-of-the-art sparse attention algorithms, including static method Sparse VideoGen (SVG)~\cite{xi2025sparsevideogenacceleratingvideo}, and dynamic methods Sparse VideoGen 2 (SVG2)~\cite{yang2025sparse} and SpargeAttention~\cite{zhang2025spargeattention}.
We do not include SLA~\cite{zhang2025sla,zhang2026sla2} because it requires extra training, which makes similarity metrics not comparable.
We provide the details of each method in Appendix \ref{app:config}.

\paragraph{Implementation.} We implement \methodname on top of SVG2~\cite{yang2025sparse}, inheriting its budget allocation strategy and default hyperparameters. Since SVG2 assigns an adaptive compute budget to each query cluster, our routing algorithm operates at the granularity of individual query clusters accordingly.
By reducing the number of centroids and the density budgets, we developed \methodname-turbo that matches the efficiency of SVG2-turbo while consistently outperforming SVG2 in quality.

\subsection{Quality Evaluation}
\begin{figure*}[t]
    \centering
    \begin{subfigure}[b]{0.59\linewidth}
        \centering
        \includegraphics[width=\linewidth]{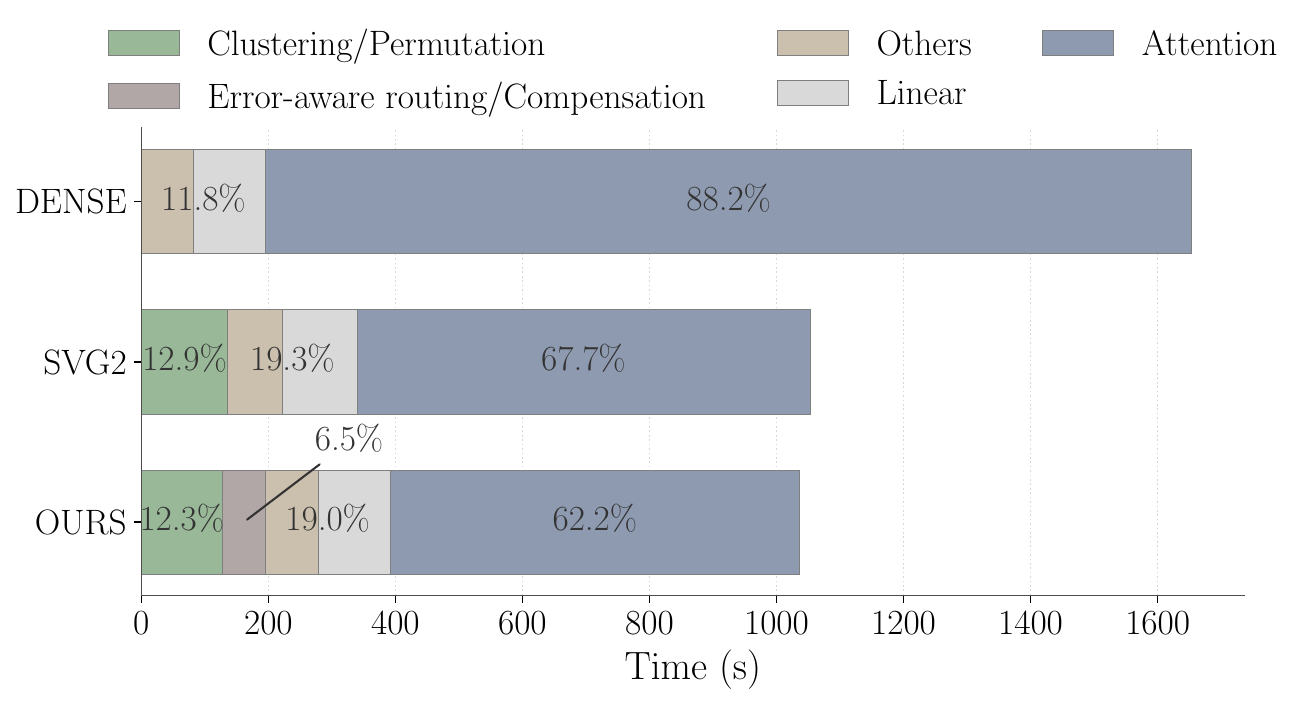}
        \caption{Generation latency breakdown}
        \label{fig:kmeans_efficiency}
    \end{subfigure}
    \hfill
    \begin{subfigure}[b]{0.40\linewidth}
        \centering
        \includegraphics[width=\linewidth]{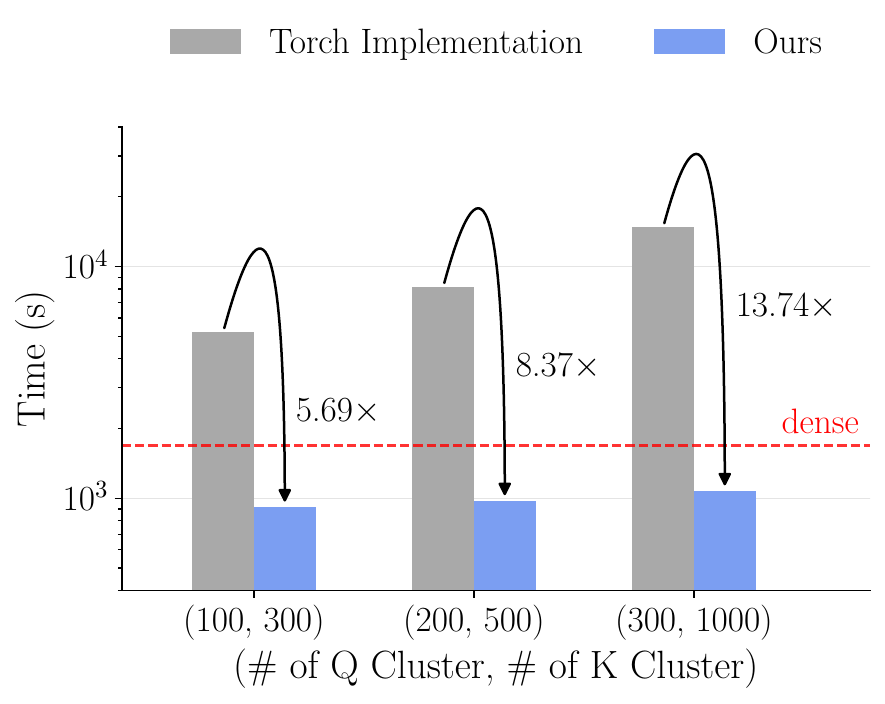}
        \caption{Efficiency of Custom Kernels}
        \label{fig:attention_efficiency}
    \end{subfigure}
    
    \caption{Efficiency evaluation. (a) illustrates the latency breakdown of different components during a single complete inference process of Wan2.2 T2V 720p, compared with the vanilla implementation and SVG2. (b) presents an end-to-end latency comparison between our efficient Triton implementation and the native PyTorch version.}
    \label{fig:overall_efficiency}
\end{figure*}

We present the attention maps of specific heads during the generation process of Wan2.2 Text-to-Video on prompts from VBench, exhibiting various sparse patterns as shown in Fig. \ref{figure:visualization}. Compared to the original top-p selection, our method covers a broader range and allocates attention to marginal tokens with relatively small attention weights. The attention maps generated by our method are more similar to the permuted maps, particularly in cases where the attention weight distribution is relatively uniform. In addition, we empirically validate the normalizer stability assumption by replacing dense attention with \texttt{SVG-EAR}, which leads to only a 4.18\% average relative change in the softmax normalizer.

The quantitative results summarized in Table \ref{table:main-results} demonstrate that \methodname consistently outperforms all baselines across PSNR, LPIPS, and SSIM, while simultaneously achieving the highest speedup. Specifically, \methodname achieves an average PSNR of \textbf{24.995} on Wan2.2 T2V and \textbf{31.043} on HunyuanVideo. Furthermore, compared to the existing SOTA SVG2~\cite{yang2025sparse}, \methodname establishes a new Pareto frontier, representing a superior trade-off between generation quality and inference efficiency. For a more comprehensive evaluation, see Appendix \ref{app:table_full}.
Moreover, for reference-free metrics, we present representative VBench metrics in Table \ref{table:main-results}, while the comprehensive evaluation is detailed in Appendix \ref{app:vbench}.

\subsection{Efficiency Evaluation}

\paragraph{End-to-end speedup evaluation.} We evaluate the efficiency of \methodname by measuring the end-to-end inference time speedup compared to full attention and total FLOPS, as shown in Table \ref{table:main-results}. \methodname achieves a speedup of \textbf{1.59$\times$} on Wan2.2 T2V and \textbf{1.93$\times$} on HunyuanVideo, higher than all baselines. Notably, \methodname-turbo achieves a speedup of \textbf{1.75$\times$} on Wan2.2 T2V, while maintaining superior quality compared to all baselines.

\paragraph{Kernel efficiency evaluation.} We evaluated the computational overhead of our method. As shown in Fig. $\ref{fig:overall_efficiency}$ (a), our approach accounts for $6.5\%$ of the total end-to-end latency during Wan2.2 text-to-video (720p) inference. 
While the latency reduction is minor, the results are noticeably better in visual fidelity and overall quality.
The efficiency is driven by our custom Triton kernel, which achieves up to a $13.74\times$ speedup compared to the original PyTorch implementation as shown in Fig. $\ref{fig:overall_efficiency}$ (b). This allows our method to enhance generation quality while maintaining a high inference speed.

\subsection{Error Analysis}
\begin{figure}[t]
  \centering
    \begin{subfigure}[b]{0.53\linewidth}
    \centering
    \includegraphics[width=\linewidth]{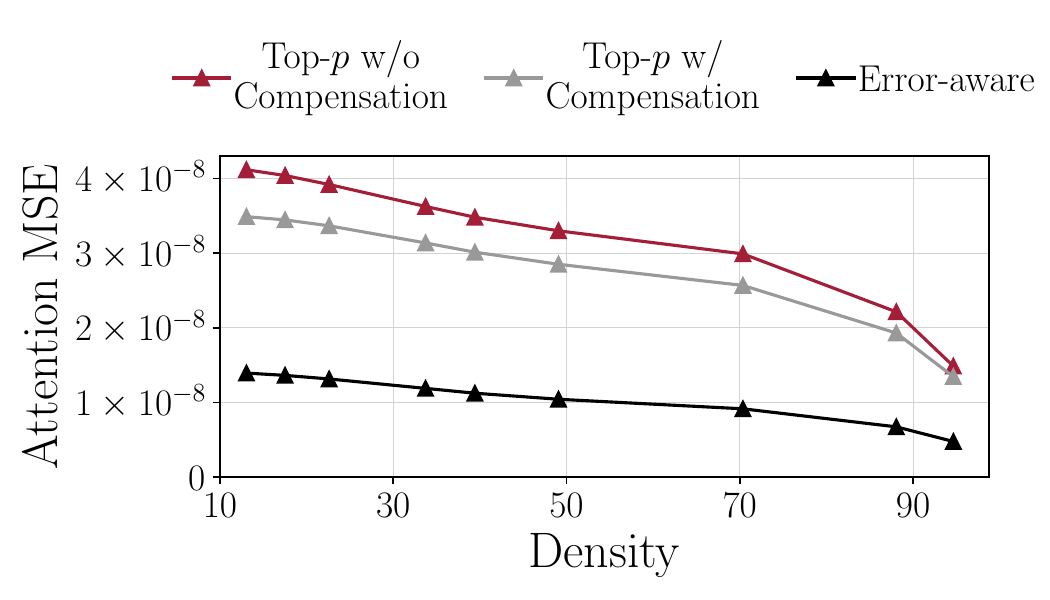}
    \caption{Attention MSE vs. Density}
    \label{fig:attention_mse}
  \end{subfigure}
  \hfill
  \begin{subfigure}[b]{0.46\linewidth}
    \centering
    \includegraphics[width=\linewidth]{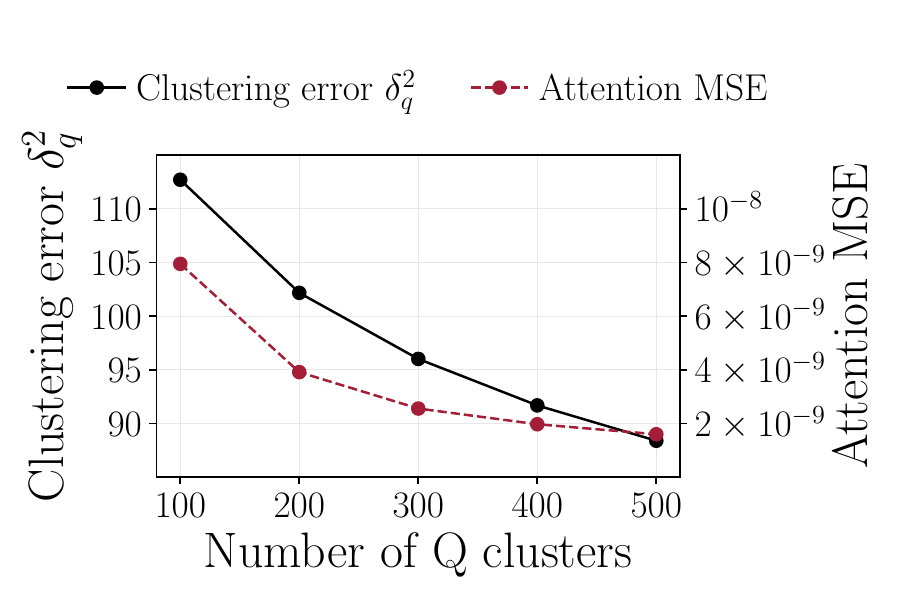}
    \caption{Clustering Error}
    \label{fig:clustering_error}
  \end{subfigure}
  \caption{Error Analysis. (a) Comparison of attention MSE versus density (\emph{i.e.,} number of sparse computation normalized by total computation) for top-$p$ selection, top-$p$ selection with linear compensation, and error-aware routing with linear compensation. (b) Both the clustering error $\delta_q^2$ (solid black line) and attention MSE (dashed red line) decrease as the number of Q clusters increases, demonstrating that higher clustering quality leads to a tighter error bound and better approximation.}
  \label{fig:attention_density_plot}
\end{figure}
\paragraph{Strategy of block selection.} To access the error-density trade-off of the three methods illustrated in Fig. \ref{fig:motivation}, \emph{i.e.,} top-$p$ selection (SVG2), top-$p$ selection with linear compensation, and error-aware routing with linear compensation, we compute the mean squared error between each method's attention map and the full attention map, \emph{i.e.,} $ \frac{1}{N_q \times N_k}\|A_{\text{sparse}} - A_{\text{full}}\|_F^2 $. Following the primary experimental settings with $Q_C=300$ and $K_C=1,000$, we conduct this evaluation by sampling a representative set of attention maps across the generation process. As shown in Fig. \ref{fig:attention_density_plot} (a), our error-aware routing yields the attention map most consistent with the full attention, while the top-$p$-based compensation only achieves a marginal error reduction against SVG2.

\paragraph{Effect of clustering.} 
We provide a theoretical guarantee that relates attention reconstruction error to clustering quality of query tokens in \S\ref{sec:methodology}.
To evaluate the effect, we conduct inference on an additional sample with $K_C=1,000$ and $p=0.85$. As shown in Fig. \ref{fig:attention_density_plot} (b), we compute the clustering error $\delta_q^2$ and the attention MSE across varying numbers of query centroids (\emph{i.e.,} $Q_C$). We observe a significant reduction in attention error as $Q_C$ increases, demonstrating that improved clustering quality directly enhances the approximation accuracy. Notably, when $Q_C$ is set to 300, the relative variation in reconstruction error remains within 4.75\%, suggesting that the approximation remains stable under this clustering configuration. This empirical observation aligns with our theoretical guarantees, where the attention error relates to the clustering quality.

\section{Conclusion and Limitation}
\label{sec:conclusion}
We presented \methodname{}, a training-free approach to accelerate DiT-based video generation under sparse attention.
\methodname{} recovers the contribution of skipped attention blocks via a parameter-free centroid compensation branch, leveraging the strong similarity structure of keys and values after semantic clustering.
To avoid the small fraction of blocks where compensation is inaccurate, we further introduced error-aware routing that prioritizes blocks with the highest error-to-cost ratio under a fixed density budget.
We provided theoretical guarantees linking attention reconstruction error to clustering quality, and demonstrated empirically that \methodname{} improves the quality--efficiency trade-off.
The main limitation of this paper is the lack of discussion and evaluation of whether the proposed method extends to attention mechanisms beyond DiTs.

\section*{Acknowledgements}
This work is supported by NSF awards IIS-1955488, IIS-2027575, DOE awards DE-SC0016260, AC02-05CH11231, and DARPA Agreement No. HR00112590131.

%
%
\bibliographystyle{splncs04}
\bibliography{main}

\clearpage
\appendix
\section{Proofs}
\label{app:bound}


\begin{proposition}
Let 
$\delta_q^2 = \frac{1}{N_q} \sum_i \| q_i - \bar{q}_i \|_2^2$
denote the average squared $\ell_2$ error between each query and its cluster mean, and let $K_{\max}$ denote the maximum $\ell_2$ norm of the key tokens. 
Given any mask $M$ that does not significantly perturb the attention normalizers.
Here, we say a mask $M$ does not significantly perturb the attention normalizers if, for all queries $i$, the normalizer remains approximately unchanged when replacing any masked entry $\exp(\frac{q_i k_j^T}{\sqrt{d}})$ (where $M^{(i)}_j = 0$) with its interpolated counterpart $\exp\Big(\frac{\big((1 - t) q_i + t_1(\bar{q}_i - q_i)\big)\big((1 - t)k_j + t_2(\bar{k}_j - k_j)\big)^T}{\sqrt{d}}\Big)$ for any $t_1, t_2 \in [0, 1]$.
We have the following error bound:
\begin{equation*}
    \frac{1}{N_qN_k} \|\tilde{A}_{\text{sparse}}^{(M)} - A_{\text{full}}\|^2_F \;\leq\; \frac{2}{N_qN_k} \sum_{i,j} (1-M^{(i)}_j)\, \frac{\hat{\epsilon}_{i,j}^2}{Z_i^2} \;+\; \frac{8\delta_q^2K_{\max}^2}{N_k d}
\end{equation*},
where $Z_i$ denotes the softmax normalizer for query $i$.
\end{proposition}

\begin{proof}
Given the normalizer stability assumption, for entries $i, j$ with $M^{(i)}_j = 0$ we have:
\begin{align}
    \Delta A_{i, j} \approx \frac{1}{Z_i} \Big(\exp\big(\frac{q_ik_j^T}{\sqrt{d}}\big) - \exp\big(\frac{q_i \bar{k}_j^T}{\sqrt{d}}\big)\Big) = \frac{f_j(q_i)}{Z_i}
\end{align},
and for entries $i, j$ with $M^{(i)}_j = 0$ we have $\Delta A_{i, j} \approx 0$.
We then have:
\begin{align}
    f_j^2(q_i) &= \big(f_j(q_i) - f_j(\bar{q}_i) + f_j(\bar{q}_i)\big)^2 \leq 2\big(f_j(q_i) - f_j(\bar{q}_i)\big)^2 + 2\big(f_j(\bar{q}_i)\big)^2 \\
    &\leq 2\big(f_j(q_i) - f_j(\bar{q}_i)\big)^2 + 2\hat{\epsilon}_{i,j}^2
\end{align}
By applying Cauchy-Schwarz inequality on the integral over $dt$, we obtain:
\begin{align}
    \label{eq:cauchy}
    2\big(f_j(q_i) - f_j(\bar{q}_i)\big)^2 \leq 2\|q_i - \bar{q}_i\|_2^2 \int_0^1 \|\nabla f_j\big(q(t)\big)\|_2^2\ dt
\end{align},
where $q(t) = (1-t)q_i + t\bar{q}_i$ is the linear interpolation between $q_i$ and $\bar{q}_i$.
Given that
\begin{align}
    \label{eq:grad}
    \nabla f_j(q(t)) = \frac{1}{\sqrt{d}} \Big(\exp\big(\frac{q(t)k_j^T}{\sqrt{d}}\big)k_j - \exp\big(\frac{q(t)k_j^T}{\sqrt{d}}\big)\bar{k_j}\Big)
\end{align}
, by applying the triangle inequality, we have:
\begin{align}
    \label{eq:triangle}
    \|\nabla f_j(q(t))\|_2 \leq \frac{K_{\max}}{\sqrt{d}} \Big(\exp\big(\frac{q(t)k_j^T}{\sqrt{d}}\big) + \exp\big(\frac{q(t)\bar{k}_j^T}{\sqrt{d}}\big)\Big)
\end{align}
Combining \eqref{eq:cauchy} and \eqref{eq:triangle}, we obtain:
\begin{align}
    2\big(f_j(q_i) - f_j(\bar{q}_i)\big)^2 \leq 2\|q_i - \bar{q}_i\|_2^2 \int_0^1 \frac{K_{\max}^2}{d} \Big(\exp\big(\frac{q(t)k_j^T}{\sqrt{d}}\big) + \exp\big(\frac{q(t)\bar{k}_j^T}{\sqrt{d}}\big)\Big)^2 dt
\end{align}
Summing over all keys $j$ with $M^{(i)}_j = 0$, we obtain:
\begin{align}
    &\frac{2}{Z_i^2}\sum_{j, M^{(i)}_j = 0} \big(f_j(q_i) - f_j(\bar{q}_i)\big)^2 \\
    &\leq 2\|q_i - \bar{q}_i\|_2^2 \int_0^1 \frac{K_{\max}^2}{d} \sum_{j, M^{(i)}_j = 0} \Big(\exp\big(\frac{q(t)k_j^T}{\sqrt{d}}\big)/Z_i^2 + \exp\big(\frac{q(t)\bar{k}_j^T}{\sqrt{d}}\big)/Z_i^2\Big)^2 dt
\end{align}
By the normalizer stability assumption, $\exp\big(\frac{q(t)k_j^T}{\sqrt{d}}\big)/Z_i^2$ and $\exp\big(\frac{q(t)\bar{k}_j^T}{\sqrt{d}}\big)/Z_i^2$ can be regarded as pesudo-probabilities $P_j(t)$ and $\bar{P}_j(t)$, respectively. This leads to:
\begin{align}
        &\frac{2}{Z_i^2}\sum_{j, M^{(i)}_j = 0} \big(f_j(q_i) - f_j(\bar{q}_i)\big)^2 \\
    &\leq 2\|q_i - \bar{q}_i\|_2^2 \int_0^1 \frac{K_{\max}^2}{d} \sum_{j} \big(P_j(t) + \bar{P}_j(t)\big)^2 dt  \\
    & \leq 2\|q_i - \bar{q}_i\|_2^2 \int_0^1 \frac{4K_{\max}^2}{d} dt = \frac{8\delta_q^2K_{\max}^2}{d}
\end{align}
By summing up over all queries $i$, we obtain:
\begin{align}
    \frac{1}{N_qN_k} \|\tilde{A}_{\text{sparse}}^{(M)} - A_{\text{full}}\|^2_F \leq \frac{2}{N_qN_k} \sum_{i,j} (1-M^{(i)}_j)\, \frac{\hat{\epsilon}_{i,j}^2}{Z_i^2} \;+\; \frac{8\delta_q^2K_{\max}^2}{N_k d}
\end{align},
which completes the proof.
\end{proof}
\section{Kernel Implementation Details}
\label{app:kernel}
To maximize hardware efficiency and minimize memory overhead, both the error estimation (Algorithm \ref{alg:error_estimation}) and fused attention (Algorithm \ref{alg:fused_qc_attention_blocked}) algorithms are implemented as highly optimized GPU kernels using Triton. For clarity, the batch and head dimensions are omitted from the pseudocode, as they can be trivially parallelized across the computation grid.

\begin{figure}[H]
    \centering
    \begin{minipage}{0.9\textwidth} 
        
        \begin{algorithm}[H]
            \caption{Error Estimation Kernel}
            \label{alg:error_estimation}
            \begin{algorithmic}[1]
                \REQUIRE $\bar{Q} \in \mathbb{R}^{Q_c \times D}$, $\bar{K}, \bar{V} \in \mathbb{R}^{K_c \times D}$
                \REQUIRE $K, V, \mathcal{S}_j$
                \ENSURE $E \in \mathbb{R}^{Q_c \times K_c}$

                \STATE $S_{i,j} \leftarrow \frac{\bar{Q}_i \bar{K}_j^\top}{\sqrt{D}}, m_i \leftarrow \max_{j} S_{i,j}$
                \STATE $Y_{i,j} \leftarrow \exp(S_{i,j} - m_i) \bar{V}_j$
                \STATE $E_{i,j} \leftarrow \|Y_{i,j}\|_2^2, m^{\text{local}}_{i,j} \leftarrow m_i$

                \FOR{each $i \in \{1, \dots, Q_c\}$}
                    \FOR{each $j \in \{1, \dots, K_c\}$}
                        \FOR{$(k, v) \in \mathcal{S}_j$}
                            \STATE $s \leftarrow \frac{\bar{Q}_i k^\top}{\sqrt{D}}$
                            \STATE $m_{\text{new}} \leftarrow \max(m^{\text{local}}_{i,j}, s)$
                            \STATE $\alpha \leftarrow \exp(m^{\text{local}}_{i,j} - m_{\text{new}})$
                            \STATE $E_{i,j} \leftarrow E_{i,j} \cdot \alpha^2, Y_{i,j} \leftarrow Y_{i,j} \cdot \alpha$
                            \STATE $E_{i,j} \leftarrow E_{i,j} + \|\exp(s - m_{\text{new}}) v - Y_{i,j}\|_2^2$
                            \STATE $m^{\text{local}}_{i,j} \leftarrow m_{\text{new}}$
                        \ENDFOR
                        \STATE $E_{i,j} \leftarrow E_{i,j} \cdot \exp(2 \cdot (m_i - m^{\text{local}}_{i,j}))$
                    \ENDFOR
                \ENDFOR
                \RETURN $E$
            \end{algorithmic}
        \end{algorithm}


        \begin{algorithm}[H]
            \caption{Fused Attention Kernel}
            \label{alg:fused_qc_attention_blocked}
            \begin{algorithmic}[1]
                \REQUIRE $Q \in \mathbb{R}^{S \times D}, \bar{K}, \bar{V} \in \mathbb{R}^{K_c \times D}$
                \REQUIRE $M \in \mathbb{R}^{Q_c \times K_c}, W \in \mathbb{R}^{K_c}, \mathcal{C}_c$
                \REQUIRE $O \in \mathbb{R}^{S \times D}, LSE \in \mathbb{R}^{S}$
                \ENSURE $\text{Out} \in \mathbb{R}^{S \times D}$

                \FOR{each cluster $c \in \{1, \dots, Q_c\}$}
                    \FOR{each query $i \in \mathcal{C}_c$}
                        \STATE $m_i \leftarrow LSE_i, \quad l_i \leftarrow 1.0$
                        \FOR{each $j$ s.t. $M_{c,j} = \text{False}$}
                            \STATE $S_{i,j} \leftarrow \frac{Q_i \bar{K}_j^\top}{\sqrt{D}} + \ln(W_j)$
                            \STATE $m_{\text{new}} \leftarrow \max(m_i, S_{i,j})$
                            \STATE $\alpha \leftarrow \exp(m_i - m_{\text{new}})$
                            \STATE $p \leftarrow \exp(S_{i,j} - m_{\text{new}})$
                            \STATE $l_i \leftarrow l_i \cdot \alpha + p$
                            \STATE $\text{acc}_i \leftarrow \text{acc}_i \cdot \alpha + p \bar{V}_j$
                            \STATE $m_i \leftarrow m_{\text{new}}$
                        \ENDFOR
                        \STATE $\text{Out}_i \leftarrow \text{acc}_i / l_i$
                    \ENDFOR
                \ENDFOR
                \RETURN $\text{Out}$
            \end{algorithmic}
        \end{algorithm}

    \end{minipage}
\end{figure}
\newpage

\section{Vbench Results}
\label{app:vbench}
To evaluate our model via the VBench framework, we focus on five key dimensions: Subject Consistency (SubConsis), Background Consistency (BackConsis), Motion Smoothness (MotionSmooth), Aesthetic Quality (AesQual), and Imaging Quality (ImagQual). All scores are averaged across our evaluation dataset, and the quantitative results are presented in Table \ref{table:vbench_results_warmup20}.






\begin{table}[H]
\centering
\caption{VBench result of \methodname.}
\label{table:vbench_results_warmup20}
\renewcommand{\arraystretch}{1.3}
\resizebox{0.95\textwidth}{!}{%
\begin{tabular}{l|ccccc}
\toprule
\textbf{Config} & SubConsis & BackConsis & MotionSmooth & AesQual & ImagQual\\
\midrule

\textit{Wan 2.2 A14B, 720P, I2V} & 0.960 & 0.960 & 0.987 & 0.628 & 0.704\\
\cmidrule(lr){1-6}
SVG      & 0.958 & \underline{0.959} & \textbf{0.989} & \underline{0.627} & \underline{0.703}\\
Sparge      & 0.958 & 0.957 & \underline{0.987} & \textbf{0.627} & 0.703\\
SVG2      & 0.958 & 0.957 & 0.986 & 0.624 & 0.701\\
\rowcolor{lightblue}
\methodname     & \textbf{0.959} & \textbf{0.959} & 0.986 & 0.627 & \textbf{0.703}\\
\rowcolor{lightblue}
\methodname-Turbo     & \underline{0.958} & 0.958 & 0.987 & 0.625 & 0.702\\

\midrule

\textit{Wan 2.1 A14B, 720P, T2V} & 0.916 & 0.941 & 0.973 & 0.650 & 0.706\\
\cmidrule(lr){1-6}
SVG          & 0.912 & 0.940 & \underline{0.974} & 0.651 & \textbf{0.712} \\ 
Sparge          & \textbf{0.916} & \underline{0.942} & \textbf{0.974} & \textbf{0.653} & \underline{0.708}\\ 
SVG2              & 0.914 & \textbf{0.943} & 0.973 & 0.651 & 0.705\\
\rowcolor{lightblue}
\methodname             & \underline{0.915} & 0.941 & 0.973 & \underline{0.652} & 0.706 \\
\rowcolor{lightblue}
\methodname-Turbo             & 0.915 & 0.941 & 0.973 & 0.651 & 0.705 \\
\midrule

\textit{Hunyuan 13B, 720P, T2V} & 0.904 & 0.945 & 0.993 & 0.625 & 0.666 \\
\cmidrule(lr){1-6}
SVG          & \underline{0.905} & \underline{0.947} & \underline{0.993} & \textbf{0.627} & \textbf{0.665} \\ 
Sparge          & \textbf{0.908} & \textbf{0.948} & \textbf{0.994} & 0.601 & 0.629 \\ 
SVG2              & 0.901 & 0.944 & 0.993 & 0.624 & 0.654 \\
\rowcolor{lightblue}
\methodname             & 0.903 & 0.944 & 0.993 & \underline{0.626} & \underline{0.659} \\

\bottomrule
\end{tabular}
}
\end{table}

\section{Config}
\label{app:config}

Table \ref{table:config} summarizes the detailed hyperparameter configurations for \methodname and the baseline methods across different tasks and model backbones. To align all implementation details except for the self-attention mechanism, we set the time warm-up of SVG to 1 to achieve a dense attention pattern. The specific parameters are defined as follows:
\begin{itemize}
\item \textbf{Time-warm} denotes the number of initial diffusion steps during which sparse attention is disabled in favor of full attention. For the majority of baseline methods, this is set to 20\% of the total timesteps . For SparseAttn and SVG on HunyuanVideo, we increase this value to 30\% to mitigate performance degradation.
\item \textbf{Layer-warm} specifies the number of model layers that utilize full dense attention during the warmup phase. For the Wan2.2 model, we apply dense attention exclusively to the first layer. Similarly, for HunyuanVideo, dense attention is restricted to the first layer across both single-stream and dual-stream blocks.
\item \textbf{TopP} is utilized to regulate the computational budget for both SVG2 and \methodname.
\item \textbf{QC} and \textbf{KC} specify the number of query and key clusters, respectively, for SVG2 and our proposed \methodname.
\item \textbf{Density} specifies the density of the attention mask, applied exclusively to SpargeAttn and SVG.
\end{itemize}

\begin{table}[H]
\centering
\caption{Detailed configuration and hyperparameter settings for \methodname and baseline methods. The symbol `` -- '' indicates that a specific parameter is not applicable to the corresponding model.}
\label{table:config}
\small
\setlength{\tabcolsep}{6pt}
\renewcommand{\arraystretch}{1.1}

\begin{tabular}{lcccccc}
\toprule
\textbf{Config} & \textbf{Time-warm} & \textbf{Layer-warm} & \textbf{QC} & \textbf{KC} & \textbf{TopP} & \textbf{Density} \\
\midrule

\multicolumn{7}{l}{\textit{Wan 2.2 A14B, 720P, image-to-video}} \\
\midrule
SpargeAttn & 10/50 & 1/40 & -- & -- & -- & 0.3 \\
SVG        & 10/50 & 1/40 & -- & -- & -- & 0.3 \\
SVG2       & 10/50 & 1/40 & 300 & 1000 & 0.9 & -- \\
SVG2-turbo       & 15/50 & 1/40 & 300 & 1000 & 0.7 & -- \\
\rowcolor{lightblue}
\methodname       & 10/50 & 1/40 & 300 & 1000 & 0.85 & -- \\
\rowcolor{lightblue}
\methodname-Turbo & 10/50 & 1/40 & 200 & 500  & 0.8 & -- \\
\midrule


\multicolumn{7}{l}{\textit{Wan 2.2 A14B, 720P, text-to-video}} \\
\midrule
SpargeAttn & 10/50 & 1/40 & -- & -- & -- & 0.3 \\
SVG        & 10/50 & 1/40 & -- & -- & -- & 0.3 \\
SVG2       & 10/50 & 1/40 & 300 & 1000 & 0.9 & -- \\
SVG2-turbo       & 15/50 & 1/40 & 100 & 500 & 0.7 & -- \\
\rowcolor{lightblue}
\methodname       & 10/50 & 1/40 & 300 & 1000 & 0.85 & -- \\
\rowcolor{lightblue}
\methodname-Turbo & 10/50 & 1/40 & 200 & 500  & 0.8 & -- \\


\midrule
\multicolumn{7}{l}{\textit{Hunyuan 13B, 720P, text-to-video}} \\
\midrule
SpargeAttn & 15/50 & 1/20, 1/40 & -- & -- & -- & 0.4 \\
SVG        & 15/50 & 1/20, 1/40 & -- & -- & -- & 0.3 \\
SVG2       & 10/50 & 1/20, 1/40 & 400 & 1000 & 0.9 & -- \\
\rowcolor{lightblue}
\methodname       & 10/50 & 1/20, 1/40 & 400 & 1000 & 0.85 & -- \\
\bottomrule
\end{tabular}
\end{table}

\section{Full Table Result}
\label{app:table_full}
In Table \ref{table:result_full}, we additionally include SVG2-Turbo, which demonstrates our clear Pareto frontier. Specifically, while SVG2-Turbo further improves acceleration at the cost of some generation quality, our turbo variant aligns with its high inference speed but achieves even higher generation quality than the baseline SVG.
\begin{table}[H]
    \centering
    \caption{Quality and efficiency benchmarking results of \methodname and baselines, where the best results are highlighted, and the second-best results are underlined.}
    \label{table:result_full}
    \renewcommand{\arraystretch}{1.5}
    \resizebox{1\linewidth}{!}{%
    \begin{tabular}{l|ccccc|ccc}
    \toprule
    \textbf{Config} & \textbf{PSNR}$\uparrow$ & \textbf{SSIM}$\uparrow$ & \textbf{LPIPS}$\downarrow$ & \textbf{ImgQual}$\uparrow$ & \textbf{SubCons}$\uparrow$ & \textbf{Density}$\downarrow$ & \textbf{FLOP}$\downarrow$ & \textbf{Speedup}$\uparrow$ \\
    \midrule

    \textit{Wan 2.2 14B, 720P, I2V} & - & - & - & 0.704 & 0.960 & 100\% & 658.46 PFLOPS & $1\times$ \\
    \cmidrule(lr){1-9}
    SpargeAttn             & 27.140 & 0.883 & 0.116 & \underline{0.703} & 0.958 & 30.15\% & 396.83 PFLOPS & $1.58\times$ \\
    SVG                    & 25.297 & 0.844 & 0.139 & 0.703 & \underline{0.958} & 30.25\% & 397.20 PFLOPS & $1.58\times$ \\
    SVG2                   & 27.668 & 0.888 & 0.117 & 0.701 & 0.958 & 29.38\% & 393.95 PFLOPS & $1.61\times$ \\
    SVG2-Turbo & 27.536 & 0.884 & 0.124 & 0.700 & 0.957 & \textbf{16.69\%} & 385.41 PFLOPS & \underline{1.72$\times$} \\
    \rowcolor{lightblue}
    \methodname     & \textbf{29.759} & \textbf{0.918} & \textbf{0.093} & \textbf{0.704} & \textbf{0.959} & \underline{23.64\%} & 378.88 PFLOPS & \underline{1.61$\times$} \\
    \rowcolor{lightblue}
    \methodname-Turbo & \underline{28.344} & \underline{0.900} & \underline{0.108} & 0.702 & 0.958 & \textbf{20.42\%} & 363.85 PFLOPS & \textbf{1.77$\times$} \\
    \midrule

    \textit{Wan 2.2 14B, 720P, T2V} & - & - & - & 0.706 & 0.916 & 100\% & 658.46 PFLOPS & $1\times$ \\
    \cmidrule(lr){1-9}
    SpargeAttn             & 20.872 & 0.708 & 0.242 & \underline{0.708} & \textbf{0.916} & 30.15\% & 396.83 PFLOPS & $1.58\times$ \\ 
    SVG      & 19.455 & 0.654 & 0.292 & \textbf{0.712} & 0.912 & 30.25\% & 397.20 PFLOPS & $1.59\times$ \\
    SVG2     & 23.556 & 0.802 & 0.183 & 0.705 & 0.914 & 32.30\% & 404.88 PFLOPS & $1.57\times$ \\
    SVG2-Turbo & 23.173 & 0.772 & 0.212 & 0.703 & 0.910 & 18.77\% & 392.23 PFLOPS & \underline{$1.71\times$} \\
    \rowcolor{lightblue}
    \methodname     & \textbf{24.995} & \textbf{0.841} & \textbf{0.153} & 0.706 & \underline{0.915} & \underline{25.95\%} & 387.53 PFLOPS & \underline{$1.59\times$} \\
    \rowcolor{lightblue}
    \methodname-Turbo & \underline{23.940} & \underline{0.814} & \underline{0.174} & 0.705 & 0.915 & \textbf{22.25\%} & 370.71 PFLOPS & \textbf{1.75$\times$} \\
    \midrule

    \textit{Hunyuan 13B, 720P, T2V} & - & - & - & 0.665 & 0.904 & 100\% & 612.38 PFLOPS & $1\times$ \\
    \cmidrule(lr){1-9}
    SpargeAttn             & 24.589 & 0.796 & 0.232 & 0.629 & \textbf{0.908} & 40.09\% & 389.76 PFLOPS & $1.38\times$ \\
    SVG      & 27.325 & 0.880 & 0.140 & \textbf{0.665} & \underline{0.905} & 29.92\% & 351.97 PFLOPS & $1.57 \times$ \\
    SVG2      & \underline{29.445} & \underline{0.911} & \underline{0.112} & 0.654 & 0.901 & 26.21\% & 299.02 PFLOPS & \underline{1.89$\times$} \\
    \rowcolor{lightblue}
    \methodname      & \textbf{31.043} & \textbf{0.928} & \textbf{0.092} & \underline{0.659} & 0.903 & 22.17\% & 281.86 PFLOPS & \textbf{1.93$\times$} \\

    \bottomrule
    \end{tabular}
    }
\end{table}

\clearpage

\section{Ablations}

We compare SVG-EAR against and SVG2 variant with our compensation module but without modifying its original block-selection strategy on Wan2.2 T2V tasks. PSNR improves from \textbf{23.556} to \textbf{24.157}, and further to \textbf{24.995} with our method. The larger gain confirms that approximation error based block selection is crucial.

\section{Sensitivity and overhead}

We additionally evaluate the effect of different clustering configurations, as shown in Table~\ref{tab:qc_kc_ablation}. All other settings are kept the same as those used in the main results. Varying $q_c/k_c$ has little effect on generation quality, while \methodname introduces only a small overhead.

\begin{table}[H]
\centering
\caption{Sensitivity and overhead under different clustering configurations.}
\label{tab:qc_kc_ablation}
\renewcommand{\arraystretch}{1.2}
\setlength{\tabcolsep}{5pt}
\begin{tabular}{ccccccc}
\toprule
$q_c$ & $k_c$ & PSNR & SSIM & LPIPS & Speedup & Overhead \\
\midrule
300 & 1000 & \textbf{24.995} & \textbf{0.841} & \textbf{0.153} & $1.59\times$ & 6.35\% \\
100 & 1000 & 23.862 & 0.806 & 0.180 & $1.68\times$ & 6.17\% \\
300 & 500  & 24.451 & 0.828 & 0.163 & $1.66\times$ & 3.82\% \\
100 & 500  & 23.438 & 0.793 & 0.191 & $1.74\times$ & \textbf{3.70\%} \\
50 & 500  & 22.501 & 0.750 & 0.228 & \textbf{$1.75\times$} & 3.74\% \\
\bottomrule
\end{tabular}
\end{table}

\section{Visualization of the Generated Videos}

We provide visualization comparison between \methodname and full Attention on HunyuanVideo and
Wan 2.2 as shown in Figure \ref{fig:wan_t2v_vis}, Figure \ref{fig:wan_i2v_vis} and Figure \ref{fig:hyvideo_t2v_vis}. 
Our generated results achieve an extremely high pixel-level similarity to full attention, achieving an excellent visual quality that is indistinguishable from the full attention results. This provides further evidence of our method's high fidelity in reproducing full attention outputs while significantly reducing computational costs.

\newpage

\begin{figure}[H]
    \centering
    
    \begin{minipage}{0.15\linewidth}
        \centering
        \small Dense Attention
    \end{minipage}
    \hfill
    \begin{minipage}{0.84\linewidth}
        \includegraphics[width=\linewidth]{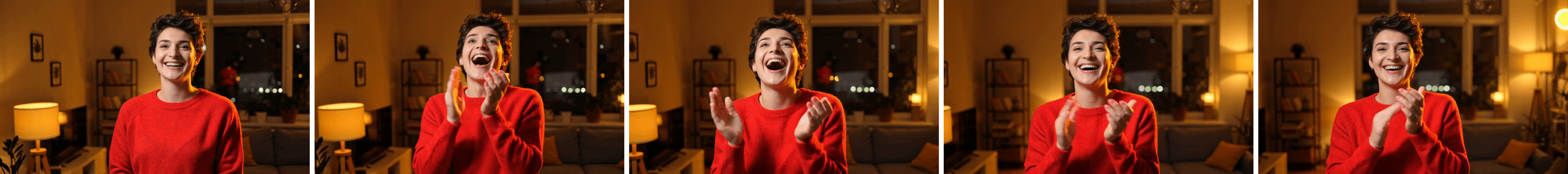}
    \end{minipage}

    \begin{minipage}{0.15\linewidth}
        \centering
        \small \methodname
    \end{minipage}
    \hfill
    \begin{minipage}{0.84\linewidth}
        \includegraphics[width=\linewidth]{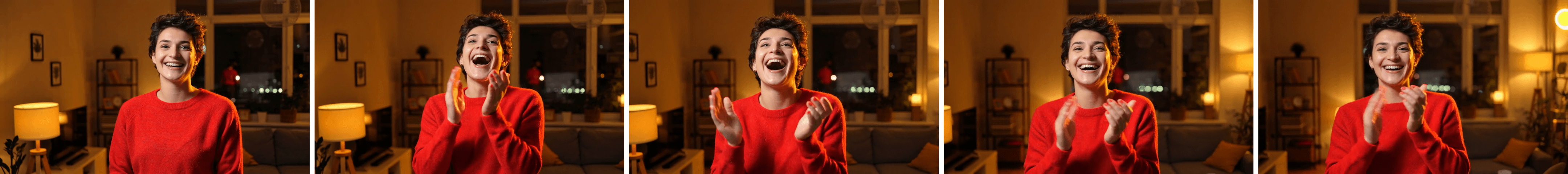}
    \end{minipage}

    \vspace{4pt} 

    \begin{minipage}{0.15\linewidth}
        \centering
        \small Dense Attention
    \end{minipage}
    \hfill
    \begin{minipage}{0.84\linewidth}
        \includegraphics[width=\linewidth]{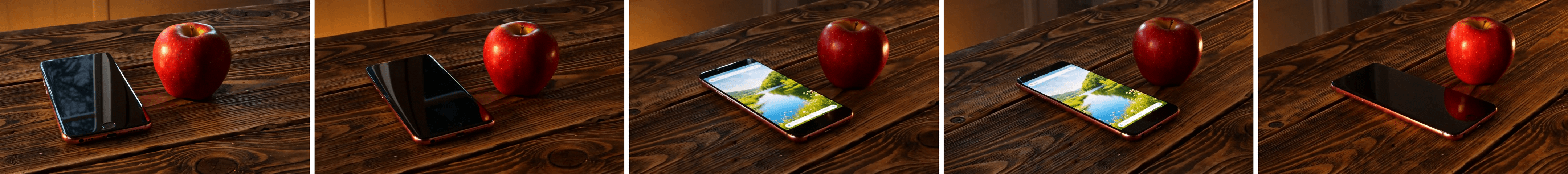}
    \end{minipage}

    \begin{minipage}{0.15\linewidth}
        \centering
        \small \methodname
    \end{minipage}
    \hfill
    \begin{minipage}{0.84\linewidth}
        \includegraphics[width=\linewidth]{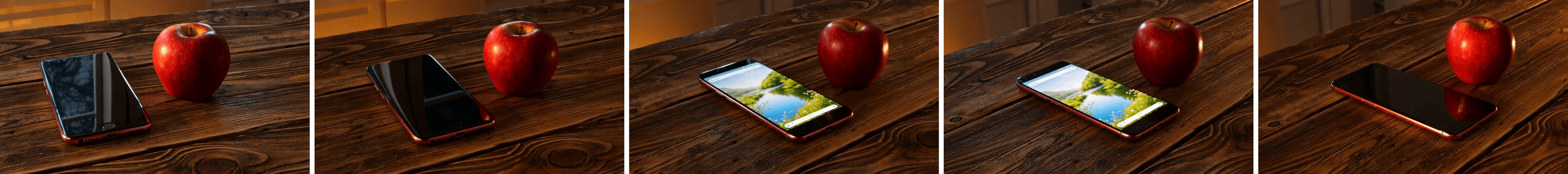}
    \end{minipage}

    \vspace{4pt}
    \begin{minipage}{0.15\linewidth}
        \centering
        \small Dense Attention
    \end{minipage}
    \hfill
    \begin{minipage}{0.84\linewidth}
        \includegraphics[width=\linewidth]{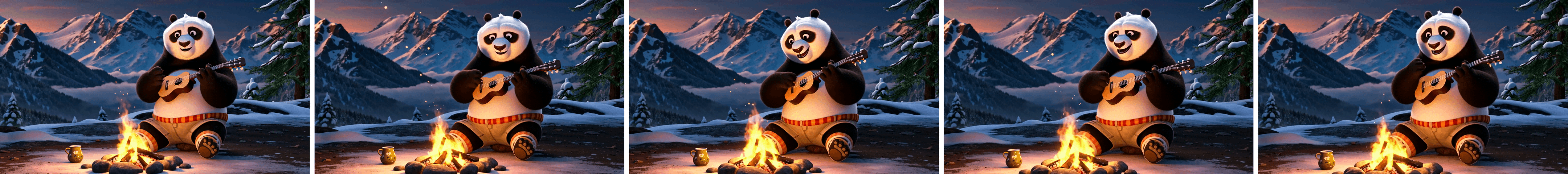}
    \end{minipage}

    \begin{minipage}{0.15\linewidth}
        \centering
        \small \methodname
    \end{minipage}
    \hfill
    \begin{minipage}{0.84\linewidth}
        \includegraphics[width=\linewidth]{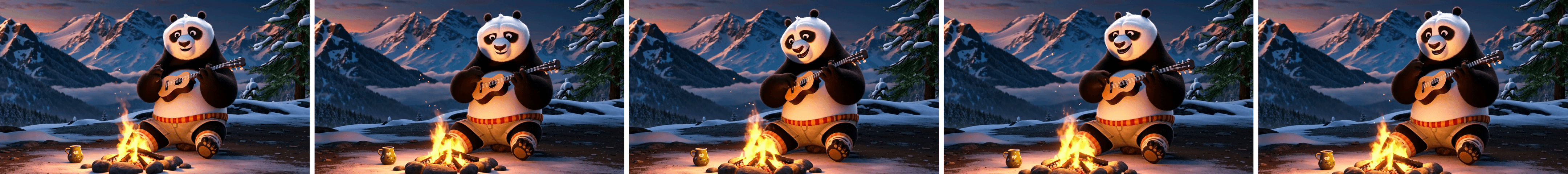}
    \end{minipage}

    \vspace{4pt} 
    \begin{minipage}{0.15\linewidth}
        \centering
        \small Dense Attention
    \end{minipage}
    \hfill
    \begin{minipage}{0.84\linewidth}
        \includegraphics[width=\linewidth]{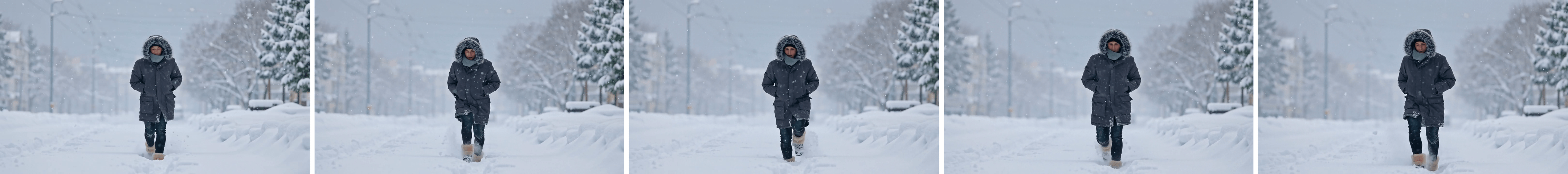}
    \end{minipage}

    \begin{minipage}{0.15\linewidth}
        \centering
        \small \methodname
    \end{minipage}
    \hfill
    \begin{minipage}{0.84\linewidth}
        \includegraphics[width=\linewidth]{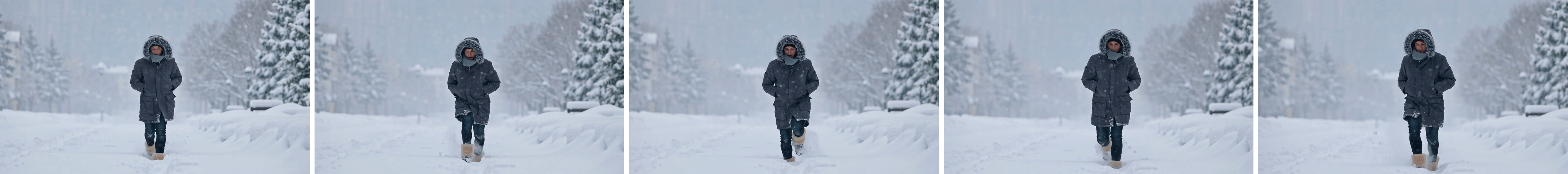}
    \end{minipage}

    \vspace{4pt} 
    \begin{minipage}{0.15\linewidth}
        \centering
        \small Dense Attention
    \end{minipage}
    \hfill
    \begin{minipage}{0.84\linewidth}
        \includegraphics[width=\linewidth]{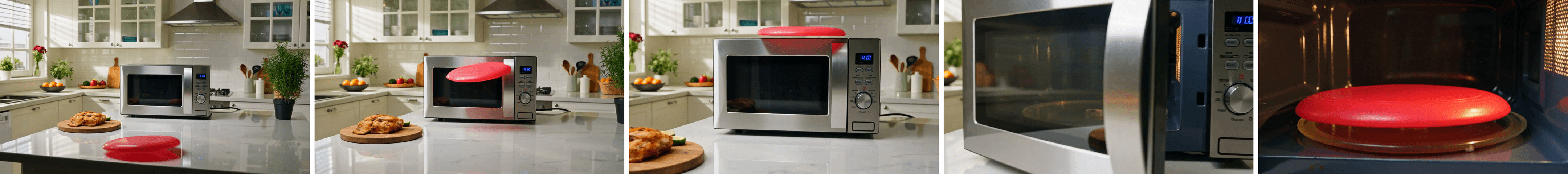}
    \end{minipage}

    \begin{minipage}{0.15\linewidth}
        \centering
        \small \methodname
    \end{minipage}
    \hfill
    \begin{minipage}{0.84\linewidth}
        \includegraphics[width=\linewidth]{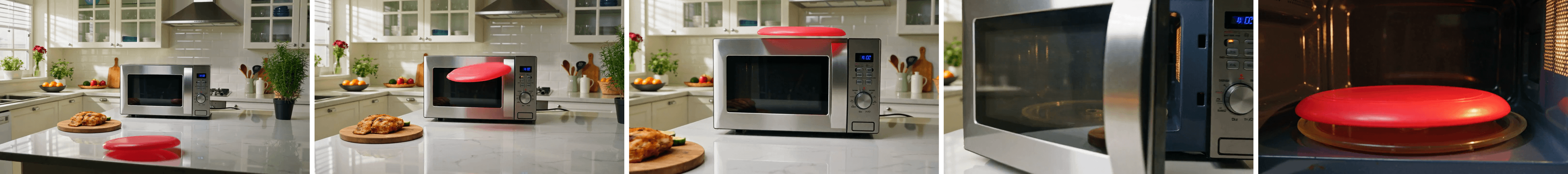}
    \end{minipage}

    \vspace{4pt} 
    \begin{minipage}{0.15\linewidth}
        \centering
        \small Dense Attention
    \end{minipage}
    \hfill
    \begin{minipage}{0.84\linewidth}
        \includegraphics[width=\linewidth]{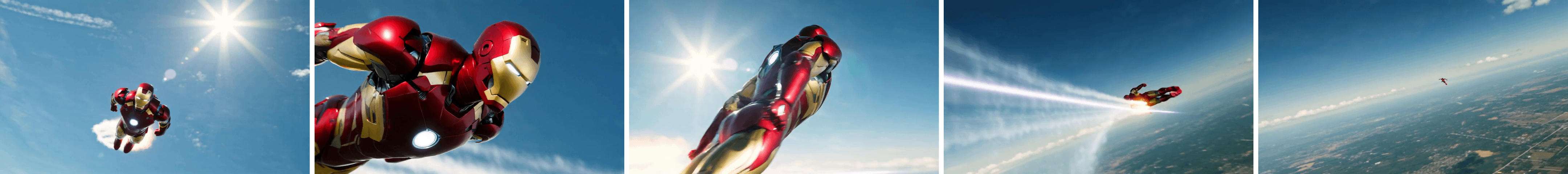}
    \end{minipage}

    \begin{minipage}{0.15\linewidth}
        \centering
        \small \methodname
    \end{minipage}
    \hfill
    \begin{minipage}{0.84\linewidth}
        \includegraphics[width=\linewidth]{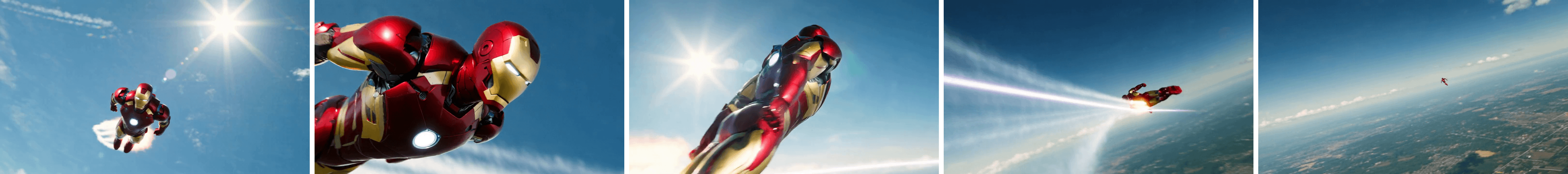}
    \end{minipage}

    \vspace{4pt} 

    \begin{minipage}{0.15\linewidth}
        \centering
        \small Dense Attention
    \end{minipage}
    \hfill
    \begin{minipage}{0.84\linewidth}
        \includegraphics[width=\linewidth]{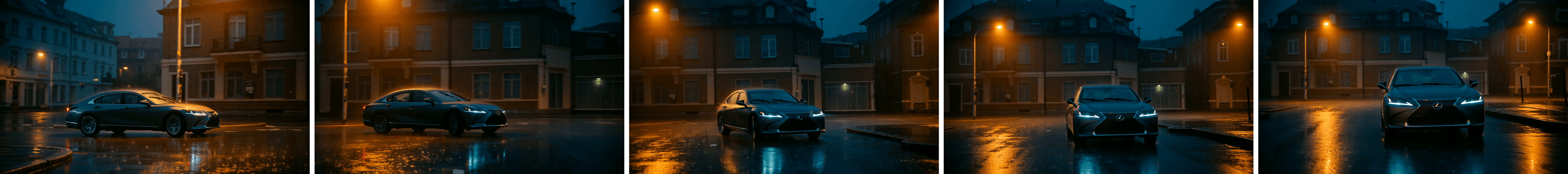}
    \end{minipage}

    \begin{minipage}{0.15\linewidth}
        \centering
        \small \methodname
    \end{minipage}
    \hfill
    \begin{minipage}{0.84\linewidth}
        \includegraphics[width=\linewidth]{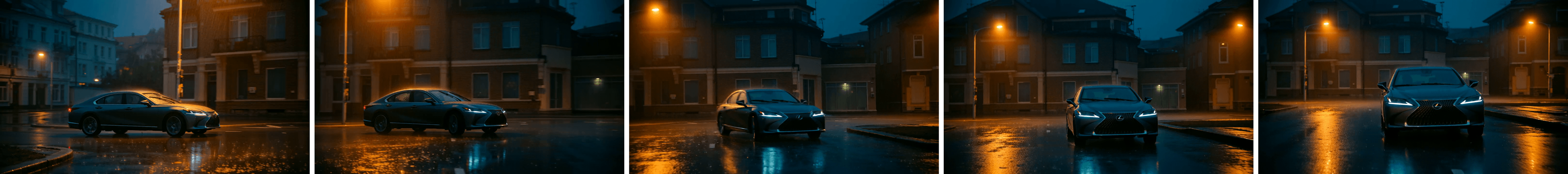}
    \end{minipage}

    \caption{Comparion of Dense Attention and \methodname on Wan 2.2 Text-to-Video generation}
    \label{fig:wan_t2v_vis}
\end{figure}

\begin{figure}[H]
    \centering
    
    \begin{minipage}{0.15\linewidth}
        \centering
        \small Dense Attention
    \end{minipage}
    \hfill
    \begin{minipage}{0.84\linewidth}
        \includegraphics[width=\linewidth]{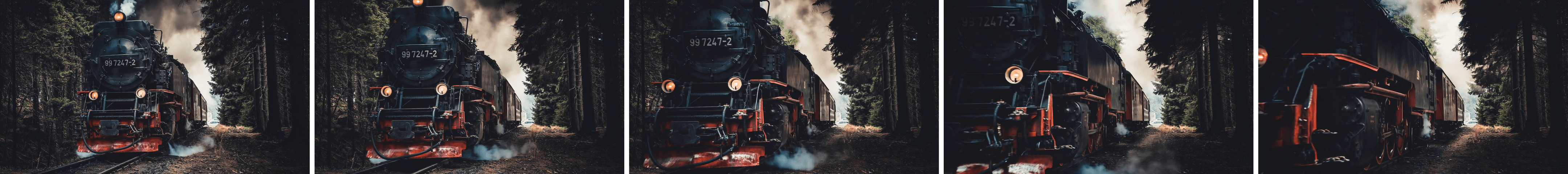}
    \end{minipage}

    \begin{minipage}{0.15\linewidth}
        \centering
        \small \methodname
    \end{minipage}
    \hfill
    \begin{minipage}{0.84\linewidth}
        \includegraphics[width=\linewidth]{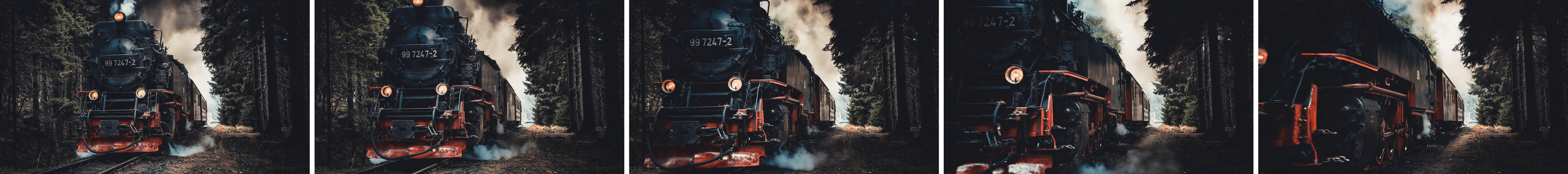}
    \end{minipage}

    \vspace{4pt} 

    \begin{minipage}{0.15\linewidth}
        \centering
        \small Dense Attention
    \end{minipage}
    \hfill
    \begin{minipage}{0.84\linewidth}
        \includegraphics[width=\linewidth]{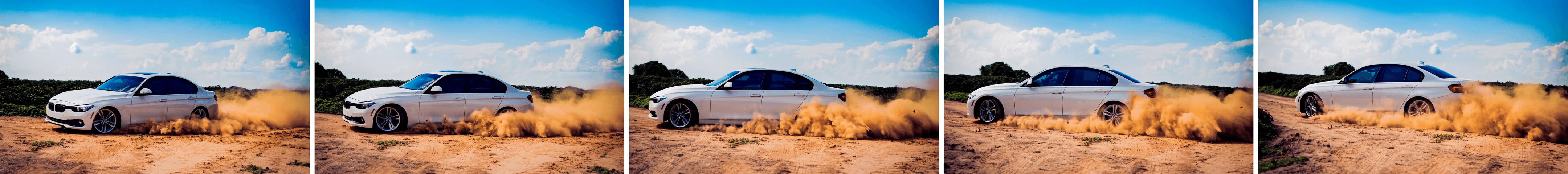}
    \end{minipage}

    \begin{minipage}{0.15\linewidth}
        \centering
        \small \methodname
    \end{minipage}
    \hfill
    \begin{minipage}{0.84\linewidth}
        \includegraphics[width=\linewidth]{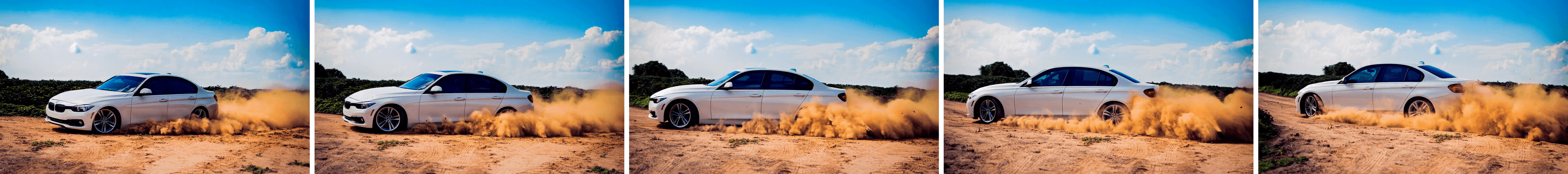}
    \end{minipage}

    \vspace{4pt}
    \begin{minipage}{0.15\linewidth}
        \centering
        \small Dense Attention
    \end{minipage}
    \hfill
    \begin{minipage}{0.84\linewidth}
        \includegraphics[width=\linewidth]{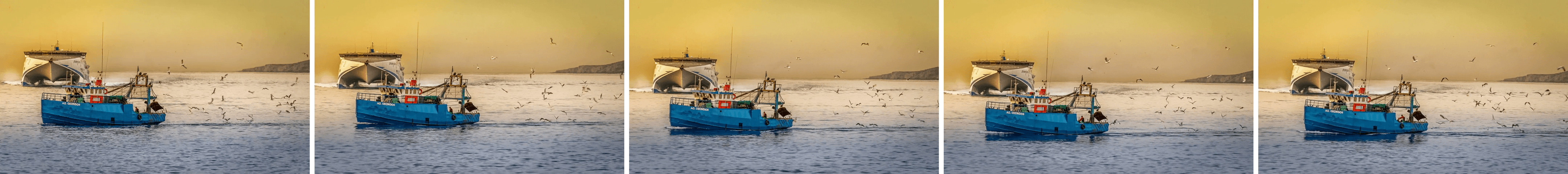}
    \end{minipage}

    \begin{minipage}{0.15\linewidth}
        \centering
        \small \methodname
    \end{minipage}
    \hfill
    \begin{minipage}{0.84\linewidth}
        \includegraphics[width=\linewidth]{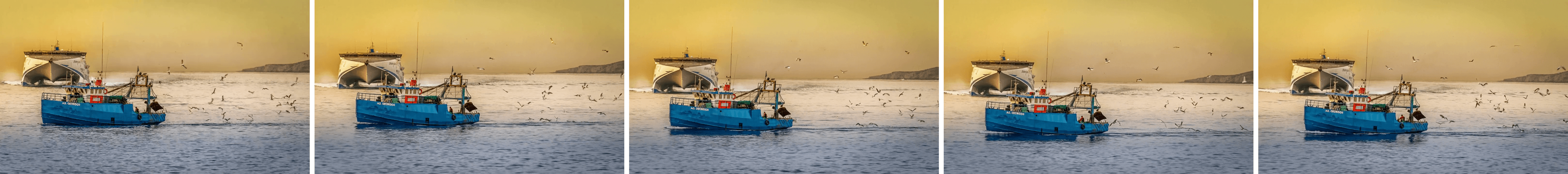}
    \end{minipage}

    \vspace{4pt} 
    \begin{minipage}{0.15\linewidth}
        \centering
        \small Dense Attention
    \end{minipage}
    \hfill
    \begin{minipage}{0.84\linewidth}
        \includegraphics[width=\linewidth]{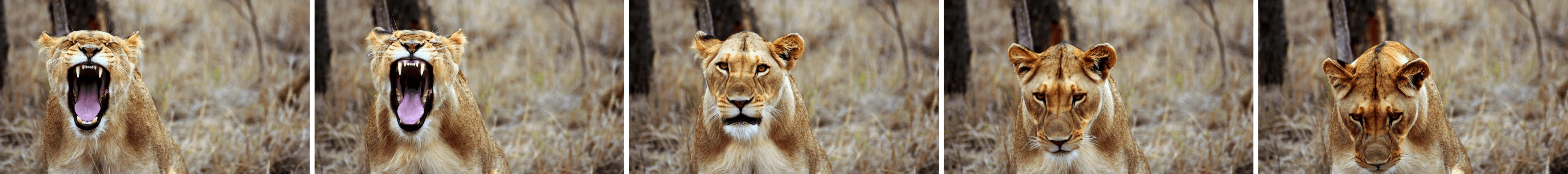}
    \end{minipage}

    \begin{minipage}{0.15\linewidth}
        \centering
        \small \methodname
    \end{minipage}
    \hfill
    \begin{minipage}{0.84\linewidth}
        \includegraphics[width=\linewidth]{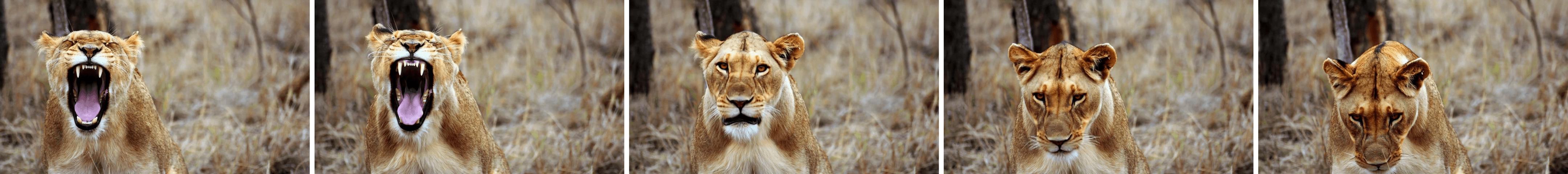}
    \end{minipage}

    \vspace{4pt} 
    \begin{minipage}{0.15\linewidth}
        \centering
        \small Dense Attention
    \end{minipage}
    \hfill
    \begin{minipage}{0.84\linewidth}
        \includegraphics[width=\linewidth]{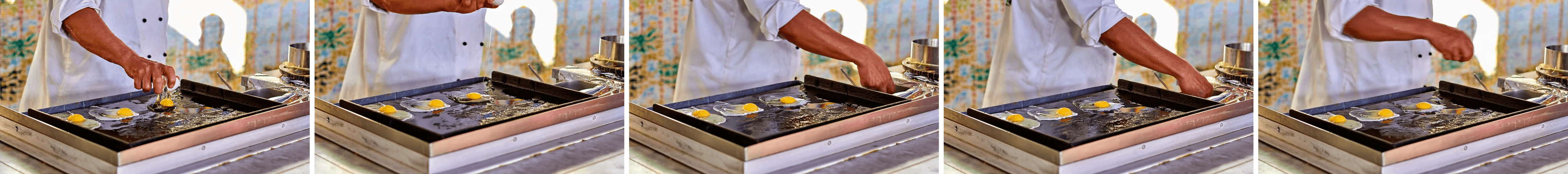}
    \end{minipage}

    \begin{minipage}{0.15\linewidth}
        \centering
        \small \methodname
    \end{minipage}
    \hfill
    \begin{minipage}{0.84\linewidth}
        \includegraphics[width=\linewidth]{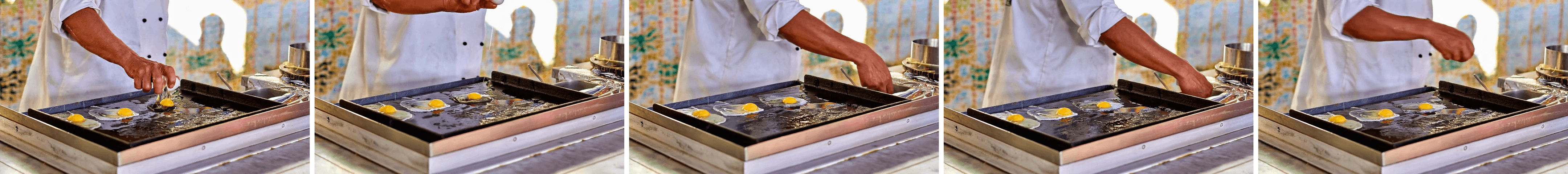}
    \end{minipage}

    \vspace{4pt} 
    \begin{minipage}{0.15\linewidth}
        \centering
        \small Dense Attention
    \end{minipage}
    \hfill
    \begin{minipage}{0.84\linewidth}
        \includegraphics[width=\linewidth]{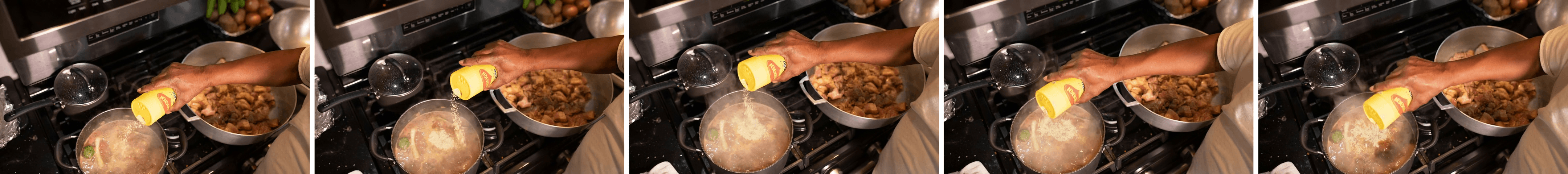}
    \end{minipage}

    \begin{minipage}{0.15\linewidth}
        \centering
        \small \methodname
    \end{minipage}
    \hfill
    \begin{minipage}{0.84\linewidth}
        \includegraphics[width=\linewidth]{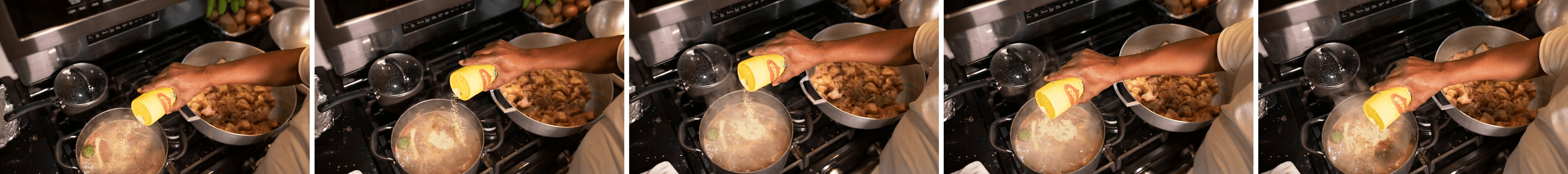}
    \end{minipage}

    \vspace{4pt} 

    \begin{minipage}{0.15\linewidth}
        \centering
        \small Dense Attention
    \end{minipage}
    \hfill
    \begin{minipage}{0.84\linewidth}
        \includegraphics[width=\linewidth]{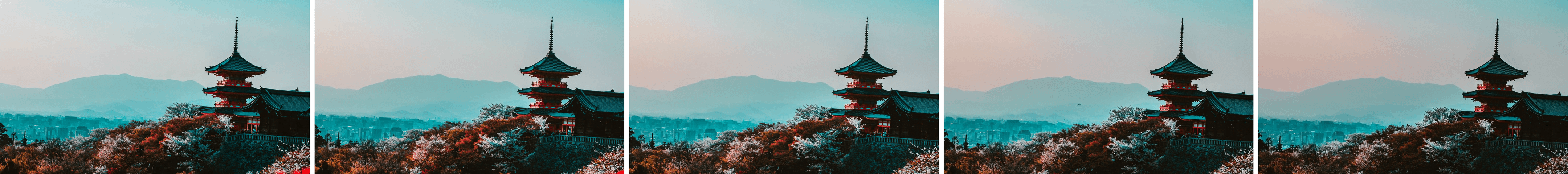}
    \end{minipage}

    \begin{minipage}{0.15\linewidth}
        \centering
        \small \methodname
    \end{minipage}
    \hfill
    \begin{minipage}{0.84\linewidth}
        \includegraphics[width=\linewidth]{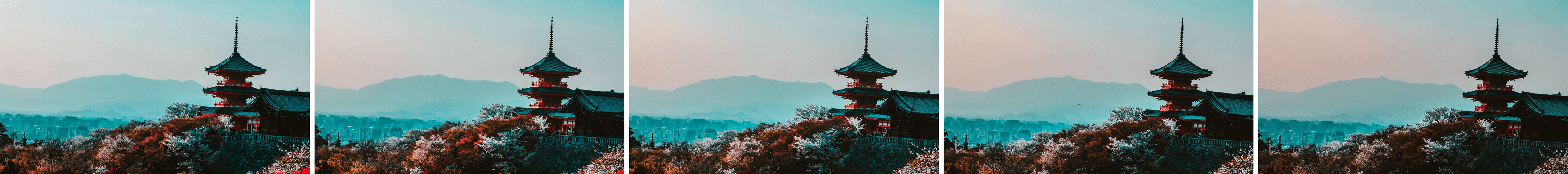}
    \end{minipage}

    \caption{Comparion of Dense Attention and \methodname on Wan 2.2 Image-to-Video generation}
    \label{fig:wan_i2v_vis}
\end{figure}

\begin{figure}[H]
    \centering
    
    \begin{minipage}{0.15\linewidth}
        \centering
        \small Dense Attention
    \end{minipage}
    \hfill
    \begin{minipage}{0.84\linewidth}
        \includegraphics[width=\linewidth]{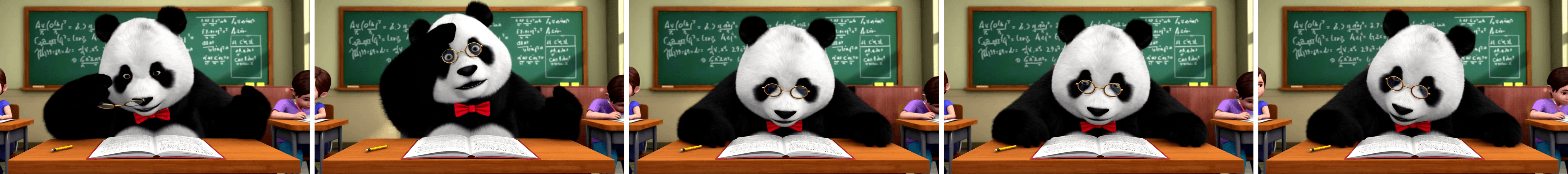}
    \end{minipage}

    \begin{minipage}{0.15\linewidth}
        \centering
        \small \methodname
    \end{minipage}
    \hfill
    \begin{minipage}{0.84\linewidth}
        \includegraphics[width=\linewidth]{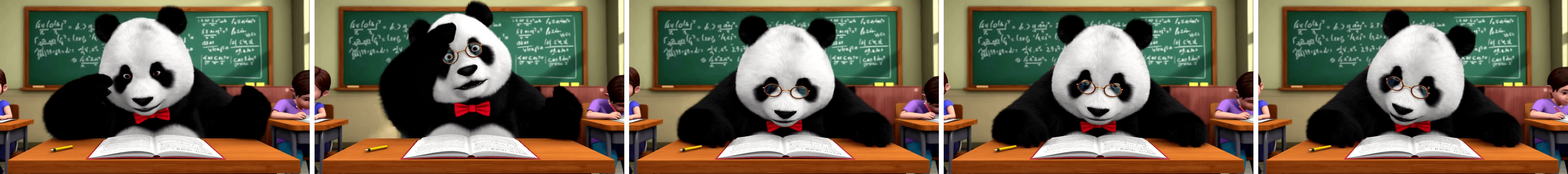}
    \end{minipage}

    \vspace{4pt} 

    \begin{minipage}{0.15\linewidth}
        \centering
        \small Dense Attention
    \end{minipage}
    \hfill
    \begin{minipage}{0.84\linewidth}
        \includegraphics[width=\linewidth]{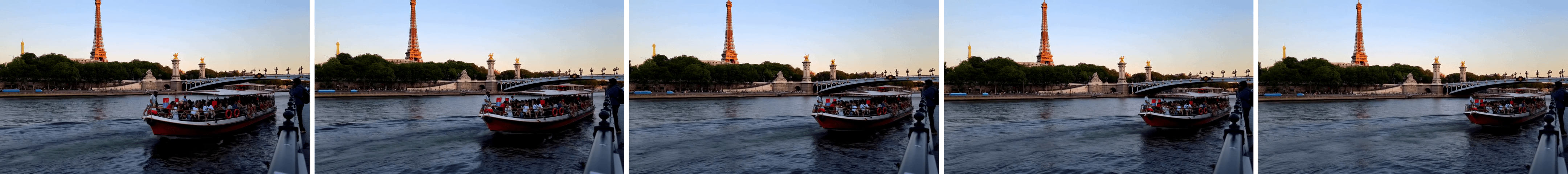}
    \end{minipage}

    \begin{minipage}{0.15\linewidth}
        \centering
        \small \methodname
    \end{minipage}
    \hfill
    \begin{minipage}{0.84\linewidth}
        \includegraphics[width=\linewidth]{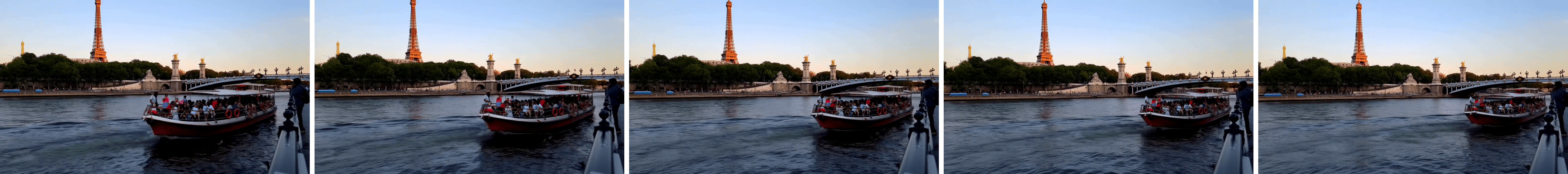}
    \end{minipage}

    \vspace{4pt}
    \begin{minipage}{0.15\linewidth}
        \centering
        \small Dense Attention
    \end{minipage}
    \hfill
    \begin{minipage}{0.84\linewidth}
        \includegraphics[width=\linewidth]{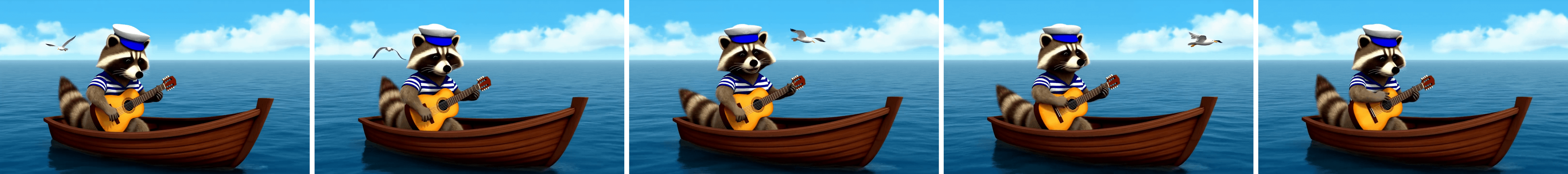}
    \end{minipage}

    \begin{minipage}{0.15\linewidth}
        \centering
        \small \methodname
    \end{minipage}
    \hfill
    \begin{minipage}{0.84\linewidth}
        \includegraphics[width=\linewidth]{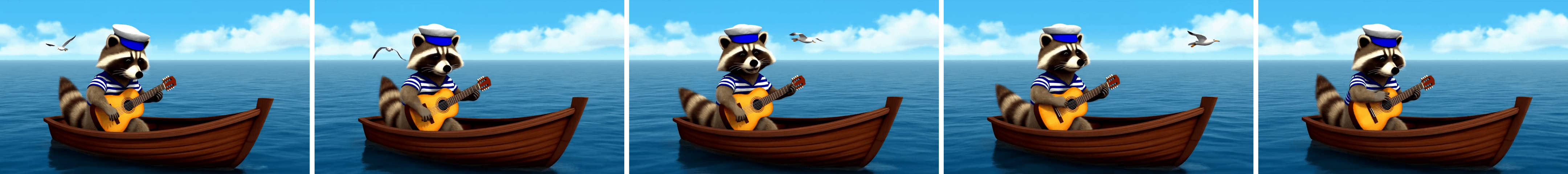}
    \end{minipage}

    \vspace{4pt} 
    \begin{minipage}{0.15\linewidth}
        \centering
        \small Dense Attention
    \end{minipage}
    \hfill
    \begin{minipage}{0.84\linewidth}
        \includegraphics[width=\linewidth]{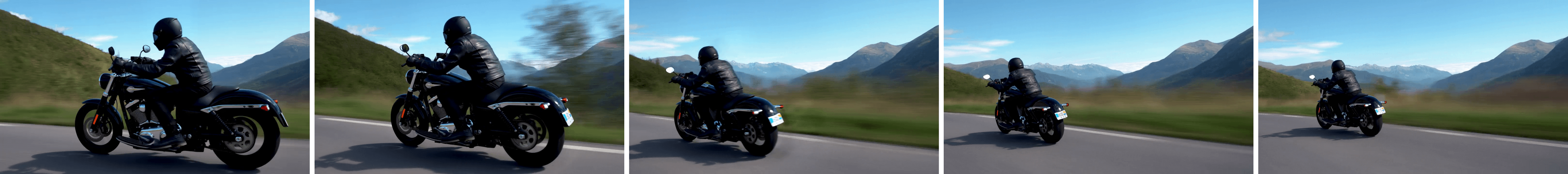}
    \end{minipage}

    \begin{minipage}{0.15\linewidth}
        \centering
        \small \methodname
    \end{minipage}
    \hfill
    \begin{minipage}{0.84\linewidth}
        \includegraphics[width=\linewidth]{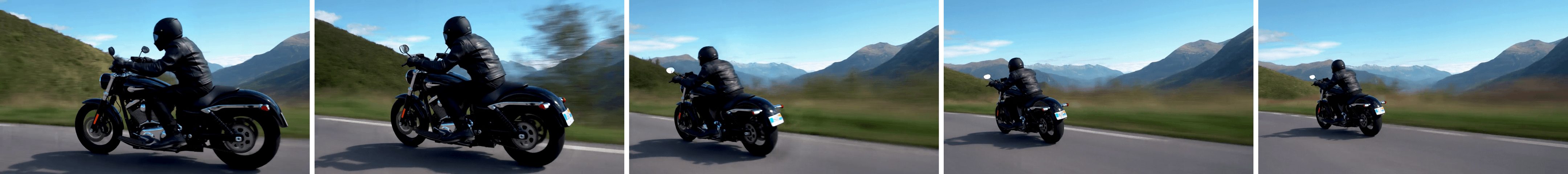}
    \end{minipage}

    \vspace{4pt} 
    \begin{minipage}{0.15\linewidth}
        \centering
        \small Dense Attention
    \end{minipage}
    \hfill
    \begin{minipage}{0.84\linewidth}
        \includegraphics[width=\linewidth]{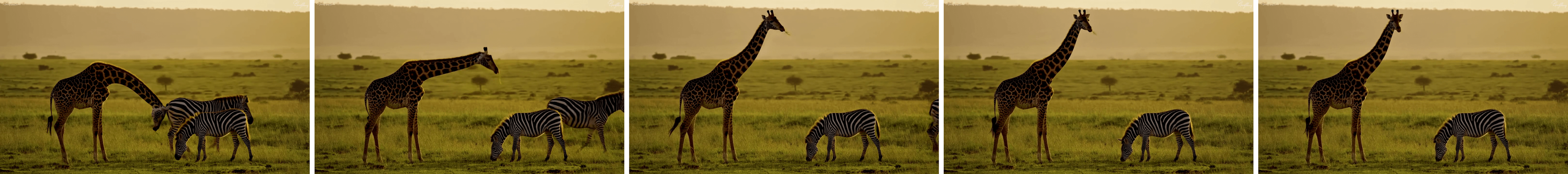}
    \end{minipage}

    \begin{minipage}{0.15\linewidth}
        \centering
        \small \methodname
    \end{minipage}
    \hfill
    \begin{minipage}{0.84\linewidth}
        \includegraphics[width=\linewidth]{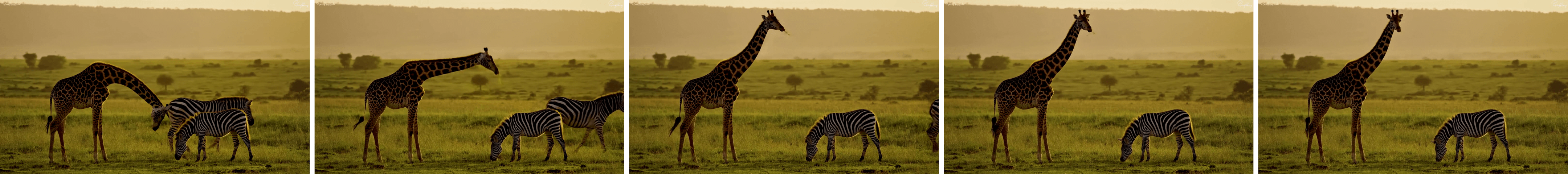}
    \end{minipage}

    \vspace{4pt} 
    \begin{minipage}{0.15\linewidth}
        \centering
        \small Dense Attention
    \end{minipage}
    \hfill
    \begin{minipage}{0.84\linewidth}
        \includegraphics[width=\linewidth]{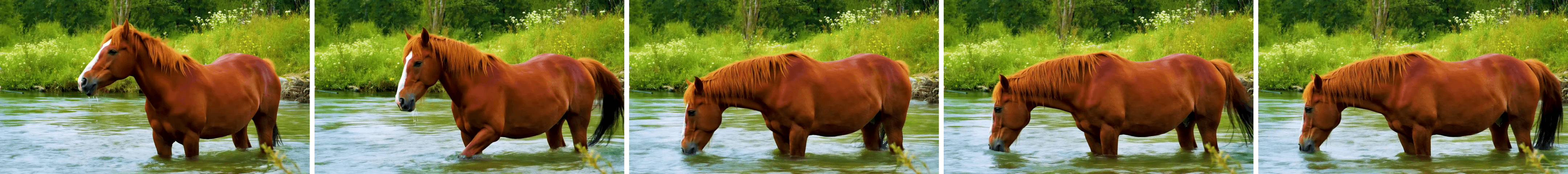}
    \end{minipage}

    \begin{minipage}{0.15\linewidth}
        \centering
        \small \methodname
    \end{minipage}
    \hfill
    \begin{minipage}{0.84\linewidth}
        \includegraphics[width=\linewidth]{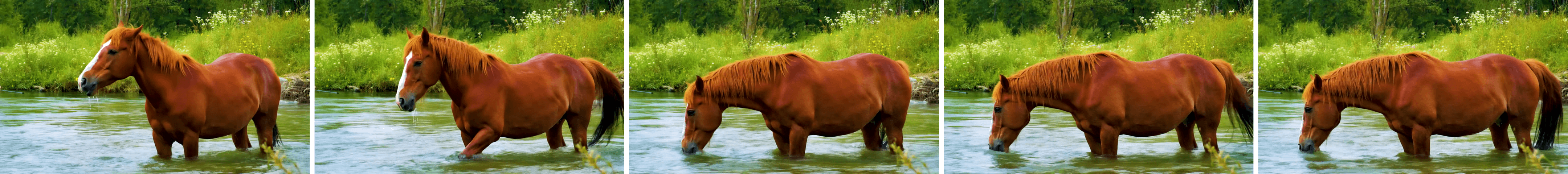}
    \end{minipage}

    \vspace{4pt} 

    \begin{minipage}{0.15\linewidth}
        \centering
        \small Dense Attention
    \end{minipage}
    \hfill
    \begin{minipage}{0.84\linewidth}
        \includegraphics[width=\linewidth]{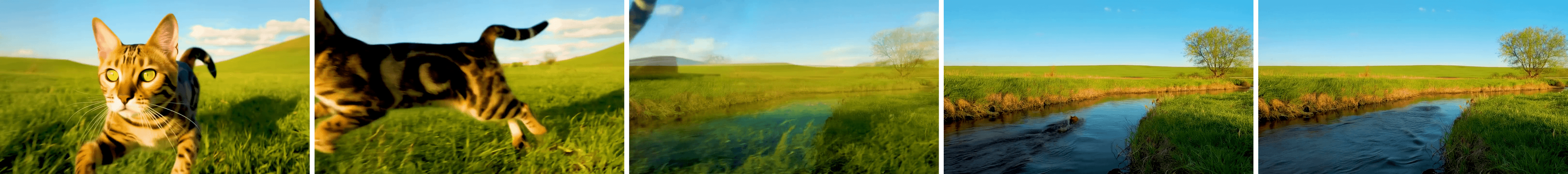}
    \end{minipage}

    \begin{minipage}{0.15\linewidth}
        \centering
        \small \methodname
    \end{minipage}
    \hfill
    \begin{minipage}{0.84\linewidth}
        \includegraphics[width=\linewidth]{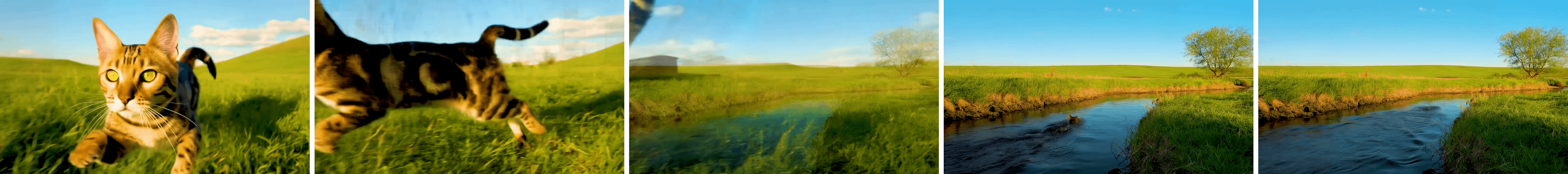}
    \end{minipage}

    \caption{Comparion of Dense Attention and \methodname on HunyuanVideo Text-to-Video generation}
    \label{fig:hyvideo_t2v_vis}
\end{figure}

\end{document}